\documentclass[letterpaper]{article}

\usepackage{aliascnt}
\usepackage{amsthm}
\usepackage[colorlinks = true, citecolor = blue]{hyperref}
\usepackage{graphicx} 
\usepackage{xcolor}
\usepackage{amsmath, amssymb}
\usepackage[algo2e, inoutnumbered, algoruled, vlined]{algorithm2e}
\usepackage[utf8]{inputenc}
\usepackage{cleveref}
\usepackage{placeins}
\usepackage{thm-restate}

\makeatletter

\let\showComments\undefined 

\ifdefined\showComments
    \newcommand{\nm}[1]{{\color{magenta}Nadav:  #1}}
    \newcommand{\hr}[1]{{\color{red}Hugo:  #1}}
    \newcommand{\pararef}[1]{(\ref{#1})}
    \newcommand\marc[1]{{\color{blue}Marc: #1}}
    \newcommand\Corentin[1]{{\color{orange}Corentin: #1}}
\else
    \newcommand{\nm}[1]{}
    \newcommand{\hr}[1]{}
    \newcommand{\pararef}[1]{}
    \newcommand\marc[1]{}
    \newcommand\Corentin[1]{}
\fi

\usepackage{bbold}
\usepackage{kky}
\usepackage{geometry}
\usepackage{algorithm}
\usepackage{subcaption}
\usepackage{comment}
\usepackage{thm-restate}
\usepackage{multirow}
\usepackage{tikz}
\usetikzlibrary{calc,arrows.meta}

\usepackage[round]{natbib}

\usepackage[accepted]{aistats2026}
\begin{document}
\raggedbottom 
\setlength{\parskip}{0pt} 
%

%

\twocolumn[
\runningauthor{Pla et al.}
\aistatstitle{On the Hardness of Reinforcement Learning with Transition Look-Ahead}

\aistatsauthor{ Corentin Pla  \And  Hugo Richard  \And  Marc Abeille  }

\aistatsaddress{ CREST, ENSAE \\ Criteo AI Lab \\ FairPlay Joint Team \And  Criteo AI Lab \\ FairPlay Joint Team \And Criteo AI Lab \\ FairPlay Joint Team } 

\aistatsauthor{  Nadav Merlis  \And  Vianney Perchet}

\aistatsaddress{ Technion \And CREST, ENSAE \\ Criteo AI Lab \\ FairPlay Joint Team  } ]

\begin{abstract}
We study reinforcement learning (RL) with transition look-ahead, where the agent may observe which states would be visited upon playing any sequence of $\ell$ actions before deciding its course of action.
While such predictive information can drastically improve the achievable performance, we show that using this information optimally comes at a potentially prohibitive computational cost.
Specifically, we prove that optimal planning with one-step look-ahead ($\ell=1$) can be solved in polynomial time through a novel linear programming formulation. In contrast, for $\ell \geq 2$, the problem becomes NP-hard. 
Our results delineate a precise boundary between tractable and intractable cases for the problem of planning with transition look-ahead in reinforcement learning.
\end{abstract}

\section{INTRODUCTION}

Reinforcement Learning (RL) \citep{sutton2018reinforcement} addresses the problem of learning how to act in a dynamic environment. This problem is modeled via a \emph{Markov Decision Process} (MDP) which involves a \emph{transition kernel}, describing how \emph{states} of the environment evolve in response to the agent’s \emph{actions}, and a \emph{reward function}, providing feedback to the agent for taking a particular action in a given state. The agent’s goal is to select actions that maximize the cumulative collected reward called \emph{return}, accounting not only for immediate gains but also for the long-term impact of its decisions on the state dynamics \citep{jaksch2010near, azar2017minimax, jin2018q}. In this work, we focus on \emph{stationary} MDP, in which the reward function and transition kernel are independent of time.

\paragraph{}In the standard RL framework, the reward and the next state are revealed only after an action has been taken. However, \emph{RL with look-ahead} assumes that, in addition to this underlying dynamical model, extra predictive information is available at decision time. In \emph{transition look-ahead}, the agent may observe before taking its action, which states would be visited upon playing any sequence of actions of length $\ell$.
This captures situations where one benefits from privileged information channels beyond standard interaction. A typical example is collaborative navigation systems that allow real-time traffic information (e.g., Waze, Coyote...) where information from nearby drivers can be used to estimate future position, speed, and traffic conditions given a sequence of routing decisions~\citep{vasserman2015implementing}. Other examples include 
access to high-fidelity but expensive simulators that can provide look-ahead on demand, or supply-chain systems where estimated delivery or arrival times are provided in advance. Standard RL algorithms do not come with off-the-shelf tools to incorporate look-ahead, and a naive policy would be to just discard this additional information. By leveraging transition look-ahead, however, the agent can anticipate the consequences of its actions before execution, enabling more effective planning while reducing uncertainty about near-term dynamics.

\paragraph{Related Work and Contribution.} 
Our main result identifies a surprising complexity threshold: planning with one-step transition look-ahead ($\ell=1$) is solvable in polynomial time, and we provide an explicit linear programming formulation. In stark contrast, planning with multi-step transition look-ahead ($\ell \geq 2$) is NP-hard. This establishes a sharp separation between tractable and intractable regimes, uncovering a fundamental frontier in the computational complexity of reinforcement learning with predictions. 

\paragraph{}
The idea of augmenting reinforcement learning with look-ahead has recently begun to attract attention.
\citet{merlis2024reinforcementlearninglookaheadinformation} introduced a pseudo-polynomial algorithm for planning with $\ell$-transition look-ahead with $\ell=1$ in the finite horizon setting.
However, in stationary MDPs, there is no polynomial-time algorithm to solve planning in the finite horizon objective~\citep{balaji2018complexity}. This makes this objective less natural for studying the hardness of look-ahead than the discounted or average objectives we consider, for which planning without look-ahead can be solved in polynomial time.  Reinforcement learning with look-ahead is also studied in~\citet{merlis2024valuerewardlookaheadreinforcement} for general $\ell \in \mathbb{N}^*$. However, the authors study \emph{reward} look-ahead, whereas we study \emph{transition} look-ahead. Furthermore, authors focus on value improvements while we study the computational complexity of planning.

\paragraph{} 
A related line of work comes from the control literature, in particular Model Predictive Control (MPC) \citep{CamachoBordons2013}.
MPC addresses the difficulty of accurately forecasting long-term system trajectories—especially under model misspecification or nonlinear dynamics—by repeatedly solving a simplified short-horizon optimization problem and then updating the control action according to the realized system state.
In this sense, short-horizon system forecasts play a role analogous to look-ahead information. However, an important distinction is that in MPC, the forecasts are obtained by simulating 
an approximate model of the system, which may be misspecified and therefore not coincide with the true dynamics.
By contrast, in our setting, transition look-ahead provides exact information about the actual next states under the environment’s true transition kernel.
Connections between MPC and reinforcement learning have been drawn in several works \citep{tamar2017learninghindsightplan, efroni2020onlineplanninglookaheadpolicies}, primarily with the goal of improving planning efficiency or coping with disturbances.
Some recent studies even analyze the dynamic regret or competitive ratio of controllers with partial look-ahead compared to those with full information \citep{li2019onlineoptimalcontrollinear, ZhangLi2021OnlineLQRwithPredictions, lin2021perturbationbasedregretanalysispredictive, lin2022boundedregretmpcperturbationanalysis}.
However, while prior studies on MPC-with-predictions analyze how well a controller performs given short-horizon forecasts, we ask a different question: how hard is it to compute an optimal plan when the planner has transition look-ahead in a discrete stationary MDP? In particular, our results prove that the tractable MPC controllers are necessarily sub-optimal. 

\paragraph{} 
Our work also connects to the broader literature on the computational complexity of solving MDPs. The foundational study of \citet{papadimitriou1987complexity} establishes, among other results, that computing optimal policies in stationary MDPs is P-complete for both discounted and average objectives. Subsequent work investigated finite-horizon MDPs in greater detail~\citep{Mundhenk2000Complexity,littman2013complexitysolvingmarkovdecision,balaji2018complexity}, with \citet{balaji2018complexity} showing EXPTIME-hardness for planning in stationary MDPs under finite-horizon objectives. More recent results also address discounted MDPs~\citep{chen2017lowerboundcomputationalcomplexity}. 

\paragraph{}
Beyond these classical formulations, several works highlight how modifications of the \emph{information structure} affect computational hardness. On the one hand, reducing information typically makes planning harder: for example, reinforcement learning with delayed feedback (where the agent receives rewards and/or transitions only after a lag) has been studied by \citet{walsh2009learning}, who show that planning in constant-delayed MDPs is NP-hard in general due to the exponential blowup of the augmented state space, while also proposing tractable algorithms for deterministic or mildly stochastic cases. More generally, POMDPs illustrate how partial observability raises the complexity to PSPACE-completeness in finite horizon and even undecidability in infinite horizon \citep{papadimitriou1987complexity}. On the other hand, our results show that \emph{increasing} the information available to the agent---by granting exact transition look-ahead---can also lead to intractability: while one-step look-ahead remains efficiently solvable, multi-step look-ahead renders the planning problem NP-hard. Thus, transition look-ahead complements the existing literature by identifying a new axis where computational complexity undergoes a phase transition: both information loss and information gain can fundamentally alter the hardness of planning.

\paragraph{} 
On the algorithmic side, linear–programming (LP)  formulations \citep{puterman2014markov} for both discounted and average–reward MDPs have been established since the foundational works of \citet{ManneLP, depenouxLP}. This line of work advocates LP methods as an alternative to dynamic–programming approaches. The distinction matters here: Value Iteration (VI) and standard Policy Iteration (PI), although widely used and efficient in practice, do not admit polynomial–time guarantees 
neither in the discounted case~\citep{feinberg2013valueiterationalgorithmstrongly, hollanders2012cdc} nor in the average case~\citep{fearnley2010}. For one–step look-ahead ($\ell=1$), our positive result gives an LP formulation that yields a polynomial-time algorithm. In sharp contrast, for $\ell\ge 2$ we prove NP–hardness, pinpointing the look-ahead depth as the tractability/intractability threshold in tabular MDPs.%

\paragraph{} 
Our setting is also related to the growing literature on algorithms with predictions \citep{mitzenmacher2020algorithmspredictions, benomar2025paretooptimalitysmoothnessstochasticitylearningaugmented,benomar2025tradeoffslearningaugmentedalgorithms, pmlr-v202-merlis23a}. These works study how providing side information can help algorithms go beyond worst-case performance, often quantifying trade-offs between consistency (when predictions are accurate) and robustness (when predictions are wrong). In reinforcement learning, this perspective has recently been explored in several directions.  \citet{NEURIPS2024_9dff3b83} design a learning-augmented controller for LQR with latent perturbations, showing that accurate predictions yield near-optimality while robustness can still be preserved under prediction errors. \citet{lyu2025efficientlysolvingdiscountedmdps} propose a framework where predictions on the transition matrix of discounted MDPs can reduce sample complexity bounds. Finally, \citet{li2023blackboxadvicelearningaugmentedalgorithms} establish a consistency–robustness tradeoff when predictions come as Q-values in non-stationary MDPs. A key distinction from our contribution is that, while these works show that predictions can improve performance beyond worst-case guarantees, we show in the context of transition look-ahead that optimally leveraging such predictions can dramatically increase the computational complexity.

\section{PROBLEM FORMULATION}
\subsection{Markov decision processes (MDP) and evaluation criteria}

We study finite tabular Markov decision processes (MDPs)
\[
\mathcal{M}=(\mathcal{S},\mathcal{A},P,r),
\]
where $\mathcal{S}$ is a finite state space, $\mathcal{A}$ is a finite action space, 
$P(s'\mid s,a)$ denotes the probability of reaching state $s' \in \mathcal{S}$ after taking action $a \in \mathcal{A}$ in state $s \in \mathcal{S}$, and
$r:\mathcal{S}\times\mathcal{A}\to[0,R_{\max}]$ is the reward function.

A (possibly randomized) stationary memoryless policy\footnote{For both the discounted and average-reward criteria considered here, one can restrict without loss of optimality to stationary memoryless policies; see, e.g., \citet{puterman2014markov}.}
is a mapping
\[
\pi:\mathcal{S}\to\Delta(\mathcal{A}),
\]
where $\Delta(\mathcal{A})$ denotes the simplex over $\mathcal{A}$. Under such a policy, the interaction evolves as follows: at each time $t\in\mathbb{N}$, the system is in state $s_t\in\mathcal{S}$, the agent selects an action $a_t\sim \pi(\cdot\mid s_t)$, receives reward $r(s_t,a_t)$, and the next state is sampled according to
\[
s_{t+1}\sim P(\cdot\mid s_t,a_t).
\]

In this paper, we focus on the discounted return and the long-run average reward, which are the canonical objectives for which planning in standard MDPs is known to be polynomial-time solvable. By contrast, note that finite–horizon objectives are less suited to our complexity analysis as there is no polynomial planning algorithm in this case, even without look-ahead information (see \citealt{balaji2018complexity}).

\paragraph{Discounted return.}
The most classical formulation assigns geometrically decaying weights to future rewards \citep[][Chapter 6]{puterman2014markov}.  For a discount factor $\gamma\in(0,1)$ and an initial state $s\in\mathcal{S}$, the value of a policy $\pi$ is
\begin{equation}
v_\gamma^\pi(s)
=
\mathbb{E}^\pi\!\left[
\sum_{t=0}^{\infty}\gamma^t r(s_t,a_t)
\,\middle|\,
s_0=s
\right].
\label{eq:eval-discounted}
\end{equation}
The optimal discounted value function is defined by
\begin{equation}
v_\gamma^*(s)
=
\sup_{\pi} v_\gamma^\pi(s),
\qquad \forall s\in\mathcal{S},
\label{eq:def-opt-discounted}
\end{equation}
where the supremum is taken over stationary memoryless policies. It is well known that $v_\gamma^*$ is the unique solution of the Bellman optimality equations $ \forall s \in \mathcal{S},$
\begin{equation}
v_\gamma^*(s)
=
\max_{a\in\mathcal{A}}
\Bigl\{
r(s,a)+\gamma\sum_{s'\in\mathcal{S}}P(s'\mid s,a)\,v_\gamma^*(s')
\Bigr\} .
\label{eq:opt-discounted}
\end{equation}
and that the optimum is attained by a stationary deterministic policy.

\paragraph{Average reward criterion.}
A second perspective focuses on the asymptotic regime, where transient effects vanish and performance is measured by the long-run average reward \citep[][Chapter 8]{puterman2014markov}. For a stationary policy $\pi$, the value for starting in $s$ is
\begin{equation*}
g^\pi(s)
=\lim_{T\to\infty}\frac{1}{T}\,
\mathbb{E}\!\left[\sum_{t=0}^{T-1} r(s_t,a_t)\;\middle|\;s_0=s,\ a_t\sim\pi(s_t)\right].
\label{eq:eval-average}
\end{equation*}
In the average case, it is standard to work under the unichain assumption: 

\begin{assumption}[Unichain MDP]
An MDP $\mathcal{M}=(\mathcal{S},\mathcal{A},P,r),$ satisfies the Unichain assumption if for any stationary policy $\pi$, for any $s\in S$, the return time 
\begin{equation*}
    \tau_s \;=\; \inf\{ t\ge 1 : S_t = s \mid S_0 = s\}
\end{equation*}
satisfies $\mathbb E_{P,\pi}[\tau_s] < \infty$.

\label{ass:unichain}
\end{assumption}

\Cref{ass:unichain} ensures that $g^\pi(s)$ does not depend on $s$, so we simply write $g^\pi$. The transient deviations from this average are measured by the \emph{bias function} $h^\pi:\mathcal{S}\to\mathbb{R}$, defined  as
\begin{equation*}
h^\pi(s) \;=\;
\mathbb{E}^\pi\!\left[\;\sum_{t=0}^{+\infty} \big( r(s_t,a_t) - g^\pi \big) \;\middle|\; s_0=s \right],
\end{equation*}
and the optimal gain $g^*$ is defined as
\begin{equation*}
g^* \;=\; \sup_{\pi}\; g^\pi.
\end{equation*}
The optimal gain/bias pair $(g^*,h^*)$ is uniquely characterized (for $h^*$, up to an additive constant) by the Bellman optimality equation: 
\begin{align*}& \forall s \in \mathcal{S}, \\&
g^*+h^*(s)=\max_{a\in\mathcal{A}}
\Big\{r(s,a)+\sum_{s'}P(s'\mid s,a)\,h^*(s')\Big\}.
\label{eq:opt-average}
\end{align*}

\medskip

\subsection{Transition look-ahead as state augmentation}

In the next section, we formalize the extra information provided by the look-ahead in terms of state observability and provide an augmented MDP construction that allows us to embed this problem into the standard evaluation framework introduced above.

\subsubsection{Look-ahead and state observability}

In the standard model, the agent only observes its current state $s_t$ before acting. We extend this by allowing the agent to query an $\ell$-step \emph{transition look-ahead}: before choosing an action, the agent is provided with the entire $\ell$-step transition tree rooted at $s_t$, i.e., all \emph{realizations of future states} that may occur within $\ell$ steps under every possible action sequence.

\paragraph{}
Let $\mathcal{M}=(\mathcal{S},\mathcal{A},P,r)$ be the MDP in which the agent is provided with $\ell$-step look-ahead information. To formalize this notion, we define the following random operator.

\begin{definition}[Push-Forward Operator]
Fix $t\in\mathbb{N}$ and an action sequence $\bar{a}_k=(a_1,\dots,a_k) \in \mathcal{A}^k$. The Push-Forward Operator, denoted by $\Pi[s_t,k,\bar{a}_k]$, is a random variable taking values in $\mathcal{S}$, corresponding to the state reached from $s_t$ after $k$ steps when actions $\bar{a}_k$ are applied.
\end{definition}

In other words, for any action sequence $\bar{a}_k$, the random variable $\Pi[s_t,k,\bar{a}_k]$ represents the state $s_{t+k}$ that would be obtained if the sequence $\bar{a}_k$ were followed from time $t$.

The look-ahead information contains all realizations of $\ell$-step trajectories and is formally defined as follows:

\begin{definition}[$\ell$-step transition look-ahead]
An agent interacting with $\mathcal{M}$ is said to be provided with $\ell$-step transition look-ahead if, for all $t \in \mathbb{N}$, it observes
\begin{equation*}
\left(\Pi[s_t, k, \bar{a}_k]\right)_{\bar{a}_k \in \mathcal{A}^k,\; 0 \le k \le \ell}
\;\in\;
\prod_{k=0}^{\ell} \mathcal{S}^{\mathcal{A}^k}.
\end{equation*} before acting. 
\end{definition}

\begin{remark}

\textbf{(i) Base case.} For all $t \in \mathbb{N}$,
\[
\Pi[s_t,0,(\emptyset)] = s_t.
\]
In particular, $0$-step transition look-ahead corresponds to the standard observation without any transition look-ahead.

\textbf{(ii) Recursion.} The collection of random variables $(\Pi[s_t,k,\bar{a}_k])$ is assumed to be \emph{prefix-consistent}: for any $k \ge 1$ and any $\bar{a}_{k+1} \in \mathcal{A}^{k+1}$,
\begin{equation}
\Pi[\Pi[s_t,1,a_1],k,\bar{a}_{2:k+1}]
=
\Pi[s_t,k+1,\bar{a}_{k+1}].
\label{eq:pi recursion}
\end{equation}

\textbf{(iii) Distribution.} For any $s \in \mathcal{S}$ and any $\bar{a}_k=(a_1,\dots,a_k)$,
\begin{align}
&\mathbb{P}\big(\Pi[s_t,k,\bar{a}_{1:k}] = s' \big) \notag 
\\&=
\mathbb{E}_{\Pi[s_t,k-1,\bar{a}_{1:k-1}]}\!\left[
P\big(s' \mid \Pi[s_t,k-1,\bar{a}_{1:k-1}], a_k\big)
\right].
\end{align}

\end{remark}

\subsubsection{Augmented MDP}
\label{subsec:augmented}

\paragraph{State and action space}
We construct an \emph{augmented MDP} $\bar{\mathcal{M}}=(\bar{\mathcal{S}},\bar{\mathcal{A}},\bar{P},\bar{r})$ such that an agent interacting with $\mathcal{M}$ with $\ell$-step transition look-ahead is equivalent to a standard agent interacting in $\bar{\mathcal{M}}$.

Let $t \in \mathbb{N}$. We define the \emph{augmented state} $\xi_t \in \bar{\mathcal{S}}$ as
\[
\bar{\mathcal{S}}
=
\prod_{k=0}^{\ell} \mathcal{S}^{\mathcal{A}^k},
\]
and
\begin{equation*}
\xi_t = \left(\xi_t[k]\right)_{0 \le k \le \ell}, 
\qquad
\text{where } \forall k \in \{0,\dots,\ell\},
\end{equation*}
\begin{equation*}
\xi_t[k] = \left(\Pi[s_t, k, \bar{a}_k]\right)_{\bar{a}_k \in \mathcal{A}^k}
\in \mathcal{S}^{\mathcal{A}^k}.
\end{equation*}

\paragraph{}
Note that although the look-ahead reveals the outcomes of \emph{hypothetical} action trajectories of length $\ell$, the agent has no incentive to commit in advance to executing the entire sequence. Indeed, if at time $t$ the agent fixes an action sequence $(a_t,\dots,a_{t+\ell-1})$, then the last $\ell-1$ actions are chosen without exploiting the additional look-ahead information that will be revealed at subsequent steps. Such a strategy therefore uses strictly less information than a policy that selects actions sequentially, conditioning on newly observed look-ahead signals. Consequently, it is without loss of optimality to restrict attention to policies that select only the current action.

The action set is therefore unchanged:
\[
\bar{\mathcal{A}} = \mathcal{A}.
\]

\paragraph{Transition and reward model}

For any $j,k,l \in \mathbb{N}$ with $j<k<l$ and any sequence $\bar{a}\in\mathcal{A}^l$, we denote $\bar a_{j:k}=(a_j,\dots,a_k)$ and $\bar a_j = \bar a_{1:j}$. We also define, for $i \le \ell$ and any prefix $\bar a'_j \in \mathcal{A}^j$,
\[
\xi_t[i](\bar a'_j)
=
\left(\Pi[s_t,i,\bar a_i]\right)_{\substack{\bar a_i \in \mathcal{A}^i \\ \bar a_j = \bar a'_j}},
\]
i.e., the restriction of the $i$-step look-ahead to action sequences with prefix $\bar a'_j$.

\paragraph{}
For any $t\in\mathbb{N}$, the first $\ell$ blocks of $\xi_t$ evolve deterministically given $(\xi_t,a_t)$, due to the consistency property of $\Pi$ (see \eqref{eq:pi recursion}). More precisely, for $k=0,\dots,\ell-1$,
\[
\xi_{t+1}[k]
=
\xi_t[k+1](a_t),
\]
where $\xi_t[k+1](a_t)$ denotes the restriction of $\xi_t[k+1]$ to action sequences whose first action is $a_t$. By contrast, the last block $\xi_{t+1}[\ell]$ is stochastic, as it corresponds to newly generated look-ahead values.

\paragraph{}
The transition kernel $\bar P$ is therefore defined as follows: for any $\xi,\xi' \in \bar{\mathcal{S}}$ and $a\in\mathcal{A}$,
\begin{align*}
&\bar P_a(\xi,\xi')=
\\& 
\mathbb{1}\{
\forall k=0,\dots,\ell-1,\;
\xi'[k]=\xi[k+1](a) \}
\\
&\quad \times
\mathbb{P}\Big(
\xi'[\ell]
=
\big(\Pi[s_{t+1},\ell,\bar a_\ell]\big)_{\bar a_\ell\in\mathcal{A}^\ell}
\;\Big|\;
\xi_t=\xi,\; a_t=a
\Big).
\end{align*}
We refer the reader to Appendix \ref{sec:transition} for detailed computation of the transition kernel. 
\paragraph{}
Finally, let $\xi \in \bar{\mathcal{S}}, a \in \mathcal{A} $ the reward function in the augmented MDP is defined by
\begin{equation*}
\bar r(\xi,a)=r(\xi[0],a),
\end{equation*}
since $\xi[0]$ encodes the current state.

\subsection{Decision problems}
To analyze the computational complexity of planning with transition look-ahead, 
we work with standard \emph{decision-problem} formulations.  
These are classical complexity–theoretic encodings of the main planning objectives 
(discounted and average reward).

\paragraph{} We focus on $\ell$-look-ahead decision problems where the agent is endowed with $\ell$--step transition look-ahead as defined
in the previous section. The look-ahead depth $\ell \in \mathbb{N}$ is thus
a fixed constant of the problem definition and not an input parameter. In particular, this implies that an algorithm that solves the decision problem after say $(\Scal \Acal)^\ell$ operations is polynomial. 
Our complexity results should be
interpreted in the same spirit as classical $k$--SAT: while $2$--SAT is
polynomially solvable, $3$--SAT is NP-hard. Analogously, we establish that
planning is tractable for $\ell=1$, but NP-hard for $\ell \ge 2$.

\paragraph{} Formally, each problem takes as input an MDP instance together with 
parameters describing the evaluation criterion, and asks whether there exists 
a (possibly randomized) policy whose value exceeds a prescribed threshold.  
Following the standards in complexity theory, the numerical values of the input, such as the discount factor $\gamma$, rewards, transition matrix, or threshold $\theta$, are encoded in binary. It implies that an algorithm with $O(\log(R_{max}))$ complexity is polynomial (where $R_{max}$ is the maximum of the reward function), but not an algorithm with complexity $O(R_{max})$.

\begin{definition}[Discounted Value Decision Problem ($\ell$-DVDP)]
\label{def:DVDP}
Instance: a finite MDP $\mathcal{M}=(\mathcal S,\mathcal A,P,r)$, an initial state $s_0\in\mathcal S$, $\gamma\in(0,1)$, and $\theta\in\mathbb R$. \\
Question: Does there exist a  policy (possibly randomized) $\pi$ such that
\begin{equation}
v^\pi_\gamma(s_0;\mathcal{M},\ell)\ \ge\ \theta\ ? 
\end{equation}
\end{definition}

\begin{definition}[Average-Reward Decision Problem ($\ell$-ARDP)]
\label{def:ARDP-body}
Instance: a finite MDP $\mathcal{M}=(\mathcal S,\mathcal A,P,r)$ and $\theta \in \mathbb{R}$. \\
Question: Does there exist a stationary (possibly randomized) policy $\pi$ such that
\begin{equation}
g^\pi(\mathcal{M},\ell)\ \ge\ \theta\ ?
\end{equation}
\end{definition}

These decision problems will serve as our canonical complexity-theoretic objects. When $\ell=0$ (no look-ahead), they are solvable in polynomial time via classical LP formulations; we will show that tractability extends to $\ell=1$, while $\ell\ge2$ renders each problem NP-hard, delineating a sharp complexity frontier for transition look-ahead.

\paragraph{} We emphasize that these decision problems are the right vehicle for hardness.
If there exists an algorithm that can compute the optimal value in polynomial time, it can be used to answer the decision question in one call by comparing its optimal value to~$\theta$. Therefore, the hardness of the decision problem implies the hardness of finding the optimal value function. 
In the other direction, an oracle that can solve the decision problems in polynomial time can be used to compute the optimal value up to any desired accuracy $\varepsilon > 0$ by bisection on $\theta$ with a complexity polynomial in the size of the input and $\log(1/\varepsilon)$. In this study, however, we solve the decision problems in the case $\ell=1$ by computing the optimal value and the optimal policy \emph{exactly}.

\section{PLANNING WITH ONE STEP TRANSITION LOOK-AHEAD}
When $\ell=1$, the look-ahead representation simplifies to observing, at time $t$ before acting in $s_t$, the collection of one-step successors $\{\Pi[s_t,1,a]\}_{a\in\mathcal A} \in \mathcal{S}$. Note that in this case, $\xi \in \bar{\mathcal{S}}$ reduces to $(s,p)$, where $s \in \mathcal{S}$ and $p \in \mathcal{S}^{\mathcal{A}}$. Therefore, the agent chooses its action after observing the entire vector of next states. In the following, we focus on planning in this setting. For clarity of exposition, we restrict attention to the discounted criterion; the arguments extend with only minor changes to the average–reward case, as will be explained at the end of this section and in more detail in \Cref{sec:discounted-case,sec:average}.

\paragraph{}
A natural way to characterize one-step look-ahead planning is to write the linear program directly in the augmented state space. In the discounted case, for $\gamma \in (0,1)$ and strictly positive weights $(\bar \mu(\xi))_{\xi \in \bar{\mathcal{S}}}$, the optimal value function $\bar v^*:\bar{\mathcal{S}} \rightarrow \mathbb{R}$ can be obtained as the solution of:
\begin{align}
& \min_{\bar v} \;\sum_{\xi \in \bar{\mathcal{S}}} \bar \mu(\xi)\, \bar v(\xi), \notag\quad  \text{s.t. }  \forall \xi \in \bar{\mathcal{S}}: \\& 
\bar v(\xi) \;\ge\; \max_{a \in \mathcal{A}} 
   \Big\{ \bar r(\xi,a) + \gamma\,\mathbb{E}_{\xi' \sim \bar{P}_a(\xi)}\!\big[ \bar v(\xi') \big] \Big\}.
\label{eq:LP}
\end{align}
Although \eqref{eq:LP} is written over the augmented state space and thus involves a value function $\bar{v}$ indexed by exponentially many augmented states, it admits an equivalent polynomial-size reformulation, which ensures tractability in the discounted setting.  

\Cref{th:1LookGamma} shows that in finite tabular MDPs, planning with one-step transition look-ahead is solvable in polynomial time for both (i) the discounted and (ii) the average-reward criteria.

\begin{restatable}[One-step look-ahead is polynomial-time]{theorem}{OneLookGamma}
\label{th:1LookGamma}
$\ell$-DVDP and $\ell$-ARDP are solvable in polynomial time for $\ell \leq 1$.
\end{restatable}

Note that the proof is constructive: we explicitly encode the planning problem as a linear program whose feasible region captures the one-step look-ahead dynamics. Solving this LP yields the optimal value and an associated optimal policy in polynomial time.

\begin{proof}[Proof sketch (discounted case)]

We consider the discounted objective with transition look-ahead of depth $\ell=1$, where at time $t$ the agent observes the entire next-state vector
\[
p\in \Scal^{\Acal}
\]
before selecting an action. In this case, an augmented state is a pair $(s,p)\in \Scal\times \Scal^{\Acal}$, where $s$ is the current state and $p(a)$ is the next state that would be reached if action $a$ were played. For each $s\in\Scal$, define the distribution
\[
Q(\cdot\mid s)\in \Delta(\Scal^{\Acal})
\]
by
\[
Q(p\mid s)=\prod_{a\in\Acal}P(p(a)\mid s,a),
\qquad p\in\Scal^{\Acal}.
\]
Thus, conditional on the current state $s$, the one-step look-ahead vector $p$ is obtained by drawing independently one successor for each action. 

We show \Cref{lem:reduced-lp} that the optimal value function $\bar v^*$ of the augmented MDP can be expressed in terms of a value function $v^*:\Scal\to\mathbb R$ as
\[
\bar v^*(s,p)
=
\max_{a\in\Acal}\{r(s,a)+\gamma v^*(p(a))\},
\ \forall (s,p)\in \Scal\times\Scal^{\Acal},
\]
where $v^*$ is the unique solution of
\begin{equation}
v(s)
=
\EE_{p\sim Q(\cdot\mid s)}
\left[
\max_{a\in\Acal}\{r(s,a)+\gamma v(p(a))\}
\right],
\ \forall s\in\Scal.
\label{eq:sketch-reduced-bellman}
\end{equation}
Moreover, for any strictly positive measure $\mu$ on $\Scal$, this fixed-point equation is equivalently characterized by the linear program
\begin{align}
\min_{v:\Scal\to\mathbb R}\quad
&(1-\gamma)\sum_{s\in\Scal}\mu(s)\,v(s)
\label{eq:sketch-reduced-lp-obj}
\quad \text{s.t. } \forall s \in \Scal \quad
\\&v(s)\ge
\EE_{p\sim Q(\cdot\mid s)}
\left[
\max_{a\in\Acal}\{r(s,a)+\gamma v(p(a))\}
\right],
\label{eq:sketch-reduced-lp-cons}
\end{align}

The remaining difficulty is that the maximization operator appears inside the expectation in \eqref{eq:sketch-reduced-lp-cons}. Expanding the expectation over all $p\in\Scal^{\Acal}$ would lead to exponentially many terms. To handle this, we use the ellipsoid method together with a polynomial-time separation oracle.

Fix $s\in\Scal$, let
\[
u_{v,s}(s',a)= r(s,a)+\gamma v(s'),
\qquad (s',a)\in \Scal\times\Acal,
\]
and let $N = |\Scal||\Acal|$. Index the pairs $(s',a)$ as $\{(s^i,a^i)\}_{i=1}^N$. For any permutation $m$ of $\{1,\dots,N\}$, define the events

\begin{align*}
    &E_i^m =
\\& 
\left\{
p\in\Scal^{\Acal}
:\;
i=\min\bigl\{j\in[N]:\, p(a^{m(j)})=s^{m(j)}\bigr\}
\right\}.
\end{align*}

These events form a partition of $\Scal^{\Acal}$. Therefore,
\begin{align*}
&\EE_{p\sim Q(\cdot\mid s)}
\left[
\max_{a\in\Acal} u_{v,s}(p(a),a)
\right]
\\&=
\sum_{i=1}^N Q(E_i^m\mid s)\, u_{v,s}(s^{m(i)},a^{m(i)})
\end{align*}

whenever $m$ orders the pairs $(s',a)$ in decreasing order of $u_{v,s}(s',a)$. This yields an equivalent family of linear constraints, $\forall s\in\Scal,\ \forall m\in\mathcal L,$
\begin{equation}
v(s)\ge
\sum_{i=1}^N Q(E_i^m\mid s)\,
\big(r(s,a^{m(i)})+\gamma v(s^{m(i)})\big),
\label{eq:sketch-sorted-cons}
\end{equation}
where $\mathcal L$ denotes the set of all permutations of $\{1,\dots,N\}$.

We then show in \Cref{lemma:polynomial} that \eqref{eq:sketch-sorted-cons} admits a polynomial-time separation oracle. Given a candidate $v$, for each state $s$ it suffices to sort the pairs $(s',a)$ in decreasing order of $u_{v,s}(s',a)$, obtaining a permutation $m_{u_{v,s}}$, and to check only the corresponding constraint. The probabilities $Q(E_i^{m_{u_{v,s}}}\mid s)$ can be computed in polynomial time using the product structure of $Q(\cdot\mid s)$.

This yields a polynomial-time separation oracle for \eqref{eq:sketch-reduced-lp-obj}--\eqref{eq:sketch-reduced-lp-cons}. By the ellipsoid method, the reduced LP can therefore be solved in polynomial time. 
\end{proof}
The average-reward case follows by the same approach; full details are given in \Cref{sec:average}.

\section{PLANNING WITH TWO OR MORE STEPS OF TRANSITION LOOK-AHEAD}
\label{sec:lge2}

We now show that allowing look-ahead of horizon $\ell \ge 2$ fundamentally changes the computational nature of planning. 

\Cref{th:LLookGamma} shows that for finite tabular MDPs, the $\ell$-step transition look-ahead planning problem is NP-hard for any $\ell\ge 2$ in the discounted setting.

\begin{restatable}[NP-hardness for $\ell\ge 2$ (discounted)]{theorem}{LLookGamma}
\label{th:LLookGamma}
 For any $\ell \geq 2$, $\ell$-\emph{DVDP} is NP-hard.
\end{restatable}

The discounted case serves as the cornerstone of our argument.  We establish NP-hardness for $\ell=2$, which extends to all larger look-ahead horizons. 

\begin{proof}
Given random variables $X_1, \dots, X_n$, an integer $k$ and a threshold $C$, the Largest Expected Value problem consists in deciding whether
\begin{equation}
    \max_{Y \subset [n] : |Y|=k} \EE[\max_{i \in Y} X_i] \geq C.
    \label{eq:mehta}
\end{equation}
Largest Expected Value has been shown to be NP-hard~\citep{mehta2020hittinghighnotessubset}. The key idea is to connect to this problem by constructing an MDP instance where computing the optimal policy requires solving a closely related subset-selection problem. More precisely, our proof relies on the gap construction of~\citet{mehta2020hittinghighnotessubset}, which reduces \textsc{Independent Set} on regular graphs to such a stochastic optimization problem.

\paragraph{}
Let $G=(V,E)$ be an undirected $3$-regular graph. Using the random variables $(X_v)_{v\in V}$ provided by the reduction of~\citet{mehta2020hittinghighnotessubset}, we construct an MDP $\mathcal{M}_G=(\Scal,\Acal,P,r,\gamma)$ whose structure is summarized in \Cref{fig:beg-three} and described in detail in \Cref{sec:proofLLookGamma}. The state space contains a root state $s_0$, a selector state $s_1$, vertex states $(s_v)_{v\in V}$, support states encoding the realizations of the random variables $(X_v)_{v\in V}$, and an absorbing terminal state $s_T$. From $s_0$, action $\textsf{wait}$ loops back to $s_0$, while action $\textsf{go}$ moves deterministically to $s_1$. From $s_1$, the actions $\textsf{pick}_1,\dots,\textsf{pick}_k$ reveal candidate vertex states, and from a vertex state $s_v$, action $\textsf{claim}$ samples the corresponding random payoff $X_v$. Rewards are only collected on the support states before reaching $s_T$.

\paragraph{}
With $\ell=2$ look-ahead, an agent at $s_0$ observes the candidate vertices that would become available after committing through $\textsf{go}$ and then choosing one of the selector actions. More precisely, each augmented root state $\xi=(s_0,p_1,p_2)$ reveals a $k$-tuple of candidate vertices, which induces a subset $S_V(\xi)\subseteq V$. If the agent plays $\textsf{wait}$, it remains in $s_0$ and resamples the two-step look-ahead, thereby drawing a fresh candidate set. If it instead commits through $\textsf{go}$, it transitions to $s_1$ and then chooses among the currently revealed candidates.

Let $X_v$ denote the random payoff obtained by claiming at vertex state $s_v$. The key structural result (\Cref{lem:root-recursion}) shows that for every augmented root state $\xi$,
\begin{align*}
&\bar v^\star(\xi)
\\&=
\max\left\{
\gamma^3\,\EE\!\left[\max_{v\in S_V(\xi)} X_v\right],
\;
\gamma\,\EE_{\xi'\sim \bar P_{\textsf{wait}}(\cdot\mid \xi)}[\bar v^\star(\xi')]
\right\}.
\end{align*}
Thus, at the root, the agent trades off immediate commitment to the currently revealed subset against waiting for a better one.

\paragraph{}
In the \textsc{NO} case, every subset of size at most $k$ has expected maximum at most $k\mu-1$, which yields a uniform upper bound on the root value. In the \textsc{YES} case, there exists a  subset $S^\star$ of size $k$ such that waiting until the look-ahead reveals exactly $S^\star$ and then committing yields a strictly larger value. Choosing $\gamma<1$ sufficiently close to $1$ separates the two cases and therefore gives a polynomial-time reduction from \textsc{Independent Set} to $2$-\emph{DVDP}. This proves NP-hardness.
\end{proof}

\begin{figure}
    \centering
    \includegraphics[width=1\linewidth]{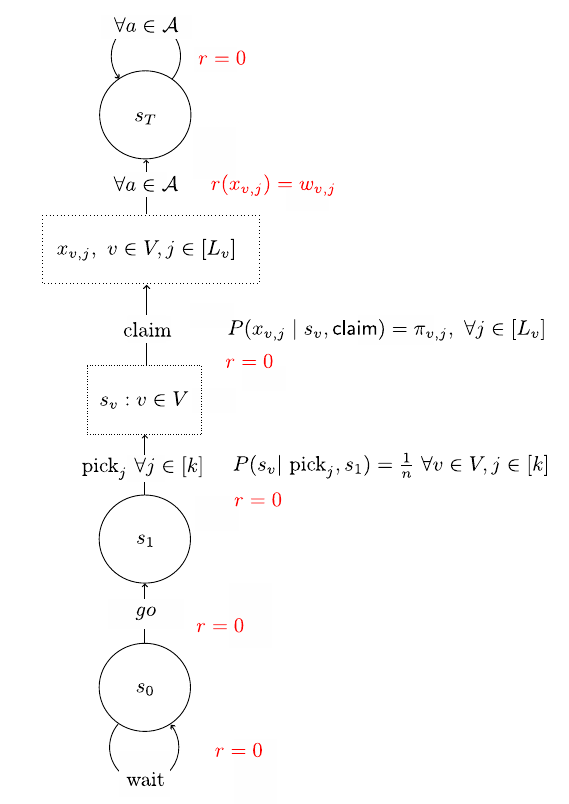}
    \caption{\textbf{Hardness of $2$-\emph{DVDP}.} At time $t=0$, the agent is at the root state $s_0$ and observes, through depth-$2$ look-ahead, the candidate vertex states that would become available after playing $\textsf{go}$ and then one of the selector actions $\textsf{pick}_1,\dots,\textsf{pick}_k$. At $s_0$, it may either play $\textsf{wait}$ to remain in $s_0$ and obtain a fresh look-ahead, or play $\textsf{go}$ to transition to $s_1$. Once in $s_1$, the agent chooses among the currently revealed candidates; if it then moves to a vertex state $s_v$, the subsequent payoff is governed by the random variable $X_v$. Thus, if the agent commits at time $\tau$, the value of the revealed look-ahead is of the form $\EE[\max_{v\in S_V(\xi_\tau)} X_v]$. Optimizing the time at which the agent commits at $s_0$ is then shown to be essentially as hard as the stochastic subset-selection problem underlying the reduction of~\citet{mehta2020hittinghighnotessubset}.}
    \label{fig:beg-three}
\end{figure}

We then show in \Cref{th:LLookAverage} that hardness also extends to the average–reward criterion. 

\begin{restatable}[NP-hardness for $\ell \ge 2$ (average reward)]{theorem}{LLookAverage}
\label{th:LLookAverage}
 For any $\ell \geq 2$, $\ell$-\emph{ARDP} is NP-hard.
\end{restatable}

\begin{proof}[Proof sketch]

We prove NP-hardness of the average–reward case by a reduction from the discounted setting (\Cref{th:LLookGamma}). Let $\bar{\mathcal M}_G$ be the augmented MDP used to prove \Cref{th:LLookGamma}. We modify $\bar{\mathcal M}_G$ by adding an independent Bernoulli ``reset'' coin at each step: with probability $1-\gamma$, the process resets to a fresh root augmented state distributed according to the look-ahead law at $s_0$, and with probability $\gamma$ it follows the original transition kernel $\bar P$. Rewards are left unchanged. Thus the dynamics of the modified MDP $\bar{\mathcal M}'_G$ are:
\begin{align*}
    \bar P'_a(\xi',\xi)
    &=
    \gamma\,\bar P_a(\xi',\xi)
    +
    (1-\gamma)\,\Lambda_{s_0}(\xi'),
\\
    \bar r'(\xi,a)&=\bar r(\xi,a).
\end{align*}

Because the reset coin is tossed independently at each step, the trajectory of $\bar{\mathcal M}'_G$ naturally decomposes into i.i.d.\ cycles between successive resets. A cycle has expected length $\tfrac{1}{1-\gamma}$, and the expected cumulative reward of a cycle under any stationary policy $\pi$ coincides with the discounted return in $\bar{\mathcal M}_G$. By the Renewal--Reward Theorem, the long-run average reward in $\bar{\mathcal M}'_G$ satisfies
\begin{equation*}
    g^\pi(\bar{\mathcal M}'_G)
    \;=\;
    (1-\gamma)\,\bar v^\pi_\gamma(\xi_0;\bar{\mathcal M}_G).
\end{equation*}

Therefore, given a threshold $\theta$ in the discounted instance, we define
\begin{equation*}
    \kappa \triangleq (1-\gamma)\theta.
\end{equation*}
Then
\begin{equation*}
    \bar v^\pi_\gamma(\xi_0;\bar{\mathcal M}_G)\ge \theta
    \;\;\Longleftrightarrow\;\;
    g^\pi(\bar{\mathcal M}'_G) \ge \kappa.
\end{equation*}
Hence, deciding whether there exists a policy exceeding $\theta$ in the discounted setting is equivalent to deciding whether there exists a policy exceeding $\kappa$ in the average–reward setting. The detailed proof is provided in Appendix~\ref{sec:ProofLLookAverage}.
\end{proof}

\section{CONCLUSION AND FUTURE WORK}

This work identifies a sharp computational frontier for planning with transition look-ahead. By introducing canonical decision formulations for both the discounted and average-reward criteria, we establish that planning is tractable in polynomial time for one-step look-ahead ($\ell=1$), with explicit linear programming formulations (\Cref{th:1LookGamma}), whereas for $\ell \ge 2$ the problem becomes NP-hard under both criteria (\Cref{th:LLookGamma,th:LLookAverage}). This dichotomy mirrors classical complexity thresholds such as the jump from $2$--SAT to $3$--SAT, and shows that while deeper look-ahead enriches the agent’s information, it simultaneously induces a combinatorial explosion that makes exact planning computationally intractable.

\paragraph{} Note that our NP-hardness results do not imply that planning with $\ell \ge 2$ is unsolvable, but rather that exact solutions cannot be expected in full generality. This motivates several directions for future work. On the approximation side, it remains open whether polynomial-time approximation schemes (PTAS) exist for discounted $\ell$–look-ahead planning, or conversely, whether even constant-factor approximation is impossible. On the structural side, one may ask under which restrictions hardness disappears: our reduction relies on constructing a worst-case instance that crucially uses dense and irregular transition structures. Hence, tractability may be recovered when the MDP satisfies additional structure, such as factored dynamics, sparse transition graphs, or monotone rewards (i.e., rewards that are non-decreasing along the natural partial order induced by the state space, which precludes the oscillatory patterns needed by our reduction). Similarly, when the discount factor $\gamma$ is sufficiently small—for instance, $\gamma < 0.5$, which dampens long-term dependencies—the reduction no longer applies, suggesting that hardness may vanish. Another direction is to study the complexity of near-optimal solutions under restricted policy classes, e.g., policies constrained by structural priors such as monotonicity or threshold rules.  Beyond exact look-ahead, a natural extension is to allow noisy or costly predictions, and to analyze how robustness and budget constraints interact with hardness. Our results also bear on learning: they confirm that model-predictive-control (MPC) style strategies—which rely on short-horizon roll-outs and commit to open-loop action sequences—are inherently suboptimal in the tabular case. This raises a broader algorithmic question: how can one design learning procedures that remain computationally efficient while approximating the (generally intractable) optimal planner with deeper look-ahead?

\ackaccepted{
This work was supported in part by the French National Research Agency (ANR) through grants ANR-20-CE23-0007 and ANR-23-CE23-0002, and through the PEPR IA FOUNDRY project (ANR23-PEIA-0003), as well as by Hi! PARIS and the ANR/France 2030 program (ANR23-IACL-0005). It was also supported by the Israel Science Foundation (ISF, Grant No.~4118/25) and the Maimonides Fund’s Future Scientists Center. The authors thank Simon Mauras for insightful discussions notably on LP reduction. We also thank the anonymous reviewers for their valuable feedback and constructive comments.}

\FloatBarrier
\bibliography{biblio}
\bibliographystyle{plainnat}

\clearpage
\section*{Checklist}

\begin{enumerate}

  \item For all models and algorithms presented, check if you include:
  \begin{enumerate}
    \item A clear description of the mathematical setting, assumptions, algorithm, and/or model. [Yes]
    \item An analysis of the properties and complexity (time, space, sample size) of any algorithm. [Yes]
    \item (Optional) Anonymized source code, with specification of all dependencies, including external libraries. [Not Applicable]
  \end{enumerate}

  \item For any theoretical claim, check if you include:
  \begin{enumerate}
    \item Statements of the full set of assumptions of all theoretical results. [Yes]
    \item Complete proofs of all theoretical results. [Yes]
    \item Clear explanations of any assumptions. [Yes]     
  \end{enumerate}

  \item For all figures and tables that present empirical results, check if you include:
  \begin{enumerate}
    \item The code, data, and instructions needed to reproduce the main experimental results (either in the supplemental material or as a URL). [Not Applicable]
    \item All the training details (e.g., data splits, hyperparameters, how they were chosen). [Not Applicable]
    \item A clear definition of the specific measure or statistics and error bars (e.g., with respect to the random seed after running experiments multiple times). [Not Applicable]
    \item A description of the computing infrastructure used. (e.g., type of GPUs, internal cluster, or cloud provider). [Not Applicable]
  \end{enumerate}

  \item If you are using existing assets (e.g., code, data, models) or curating/releasing new assets, check if you include:
  \begin{enumerate}
    \item Citations of the creator If your work uses existing assets. [Not Applicable]
    \item The license information of the assets, if applicable. [Not Applicable]
    \item New assets either in the supplemental material or as a URL, if applicable. [Not Applicable]
    \item Information about consent from data providers/curators. [Not Applicable]
    \item Discussion of sensible content if applicable, e.g., personally identifiable information or offensive content. [Not Applicable]
  \end{enumerate}

  \item If you used crowdsourcing or conducted research with human subjects, check if you include:
  \begin{enumerate}
    \item The full text of instructions given to participants and screenshots. [Not Applicable]
    \item Descriptions of potential participant risks, with links to Institutional Review Board (IRB) approvals if applicable. [Not Applicable]
    \item The estimated hourly wage paid to participants and the total amount spent on participant compensation. [Not Applicable]
  \end{enumerate}

\end{enumerate}

\appendix
\onecolumn

\theoremstyle{plain}
\newtheorem{unlemma}{Lemma S}
\newtheorem{unproposition}{Proposition S}
\newtheorem{uncorollary}{Corollary S}
\newtheorem{untheorem}{Theorem S}

\setcounter{equation}{0}
\setcounter{figure}{0}
\setcounter{table}{0}
\setcounter{page}{1}
\makeatletter
\renewcommand{\theequation}{S\arabic{equation}}
\renewcommand{\thefigure}{S\arabic{figure}}
\renewcommand{\thetable}{S\arabic{table}}
\renewcommand{\thetheorem}{S\arabic{theorem}}
\renewcommand{\thelemma}{S\arabic{lemma}}
\renewcommand{\thesection}{S\arabic{section}}
\renewcommand{\theremark}{S\arabic{remark}}
\renewcommand{\theproposition}{S\arabic{proposition}}
\renewcommand{\thecorollary}{S\arabic{corollary}}

\clearpage

\section{Transition dynamics}
\label{sec:transition}

In this section, we make explicit the transition kernel of the augmented MDP
\[
\bar{\mathcal M} = (\bar{\mathcal S}, \mathcal A, \bar P, \bar r)
\]
introduced in \cref{subsec:augmented}.

Recall that, at time $t$, the augmented state $\xi_t \in \bar{\mathcal S}$ encodes the full $\ell$-step look-ahead tree available to the agent from the current state $s_t$. More precisely, for each depth $k \in [\ell]$, the component
\[
\xi_t[k] : \mathcal A^k \to \mathcal S
\]
maps every action sequence $a_{1:k} \in \mathcal A^k$ to the state reached after applying that sequence from $s_t$. We write
\[
\xi_t[k](a_{1:k}) = \Pi[s_t,k,a_{1:k}],
\]
where $\Pi[s,k,a_{1:k}]$ denotes the random state obtained by starting from $s$ and applying the action sequence $a_{1:k}$ for $k$ steps.

The augmented transition kernel is defined by
\[
\bar P_a(\xi,\xi')
\triangleq
\mathbb P(\xi_{t+1}=\xi' \mid \xi_t=\xi,\ a_t=a),
\qquad
\xi,\xi' \in \bar{\mathcal S},\ a\in\mathcal A.
\]

The key observation is that, after playing action $a$ at time $t$, the first $\ell-1$ levels of the new look-ahead tree are fully determined by the previous one. Indeed, for every $k\in\{1,\dots,\ell-1\}$ and every $a_{1:k}\in\mathcal A^k$,
\[
\xi_{t+1}[k](a_{1:k})
=
\xi_t[k+1](a,a_{1:k}).
\]
This follows from the recursion property
\[
\Pi[\Pi[s,1,a],k,a_{1:k}]
=
\Pi[s,k+1,(a,a_{1:k})].
\]

Hence, conditional on $(\xi_t=\xi,a_t=a)$, the components
$\xi_{t+1}[1],\dots,\xi_{t+1}[\ell-1]$
are deterministic, and only the last block $\xi_{t+1}[\ell]$ remains random.

It follows that
\[
\bar P_a(\xi,\xi')
=
\mathbb 1\!\left\{
\forall k\in\{1,\dots,\ell-1\},\ 
\forall a_{1:k}\in\mathcal A^k,\ 
\xi'[k](a_{1:k})=\xi[k+1](a,a_{1:k})
\right\}
\cdot
\mathbb P\!\left(\xi_{t+1}[\ell]=\xi'[\ell]\mid \xi_t=\xi,\ a_t=a\right).
\]
We now characterize the second factor.

Fix $\xi\in\bar{\mathcal S}$ and $a\in\mathcal A$. For each prefix
$a_{1:\ell-1}\in\mathcal A^{\ell-1}$, define
\[
s_\xi(a_{1:\ell-1}) \triangleq \xi[\ell](a,a_{1:\ell-1}).
\]
This is the state reached after first taking action $a$, and then following the continuation $a_{1:\ell-1}$ inside the current look-ahead tree.

Then, by the Markov property, for every terminal action $a_\ell\in\mathcal A$,
\[
\xi_{t+1}[\ell](a_{1:\ell})
\sim
P(\cdot \mid s_\xi(a_{1:\ell-1}), a_\ell).
\]
Therefore, the law of the last block is exactly the joint law of all one-step successors from the frontier states of the current tree. Equivalently,
\begin{align}
&\mathbb P\!\left(\xi_{t+1}[\ell]=\xi'[\ell]\mid \xi_t=\xi,\ a_t=a\right)
\nonumber\\
&\qquad =
\mathbb P\!\left(
\forall a_{1:\ell}\in\mathcal A^\ell,\ 
\Pi\!\left[s_\xi(a_{1:\ell-1}),1,a_\ell\right]
=
\xi'[\ell](a_{1:\ell})
\right).
\label{eq:last-block-law}
\end{align}

The structure of the last block is constrained by a consistency requirement: multiple indices correspond to the same underlying random variable. 

More precisely, for any $s\in\Scal$, all prefixes $\bar a_{\ell-1}\in\Acal^{\ell-1}$ such that
\[
\xi[\ell](a,\bar a_{\ell-1})=s
\]
lead to the same frontier state $s$, and therefore share the same one-step successor $\Pi[s,1,a']$ for any $a'\in\Acal$. 

This induces the following partition of $\Acal^{\ell-1}$:
\[
\Scal_{a}^{\ell-1}(\xi)
=
\left\{
s \in \Scal : \exists \bar a_{\ell-1} \in \Acal^{\ell-1},\ 
\xi[\ell](a,\bar a_{\ell-1})=s
\right\},
\]
\[
\mathcal{A}^{\ell-1}_{s,a}(\xi)
=
\left\{
\bar a_{\ell-1} \in \Acal^{\ell-1} : 
\xi[\ell](a,\bar a_{\ell-1})=s
\right\}.
\]

Using this partition, we can rewrite
\begin{align}
&\mathbb P\!\left(\xi_{t+1}[\ell]=\xi'[\ell]\mid \xi_t=\xi,\ a_t=a\right)
\nonumber\\
&\qquad =
\mathbb P\!\left(
\forall s\in\Scal_{a}^{\ell-1}(\xi),\ 
\forall \bar a_{\ell-1}\in\mathcal{A}^{\ell-1}_{s,a}(\xi),\ 
\forall a'\in\Acal,\ 
\Pi[s,1,a']=\xi'[\ell](\bar a_{\ell-1},a')
\right).
\label{eq:last-block-grouped}
\end{align}

For a fixed pair $(s,a')$, all indices $(\bar a_{\ell-1},a')$ with 
$\bar a_{\ell-1}\in\mathcal{A}^{\ell-1}_{s,a}(\xi)$ correspond to the same random variable $\Pi[s,1,a']$. Hence, the probability in \eqref{eq:last-block-grouped} is zero unless, for every $s\in\Scal_{a}^{\ell-1}(\xi)$ and every $a'\in\Acal$, the value
\[
\xi'[\ell](\bar a_{\ell-1},a')
\]
is constant over all $\bar a_{\ell-1}\in\mathcal{A}^{\ell-1}_{s,a}(\xi)$.

When this consistency condition holds, the probability factorizes over distinct pairs $(s,a')$, yielding
\begin{align}
&\mathbb P\!\left(\xi_{t+1}[\ell]=\xi'[\ell]\mid \xi_t=\xi,\ a_t=a\right)
\nonumber\\
&\qquad =
\prod_{s\in\Scal_{a}^{\ell-1}(\xi)}
\prod_{a'\in\Acal}
\left(
\sum_{s'\in\Scal}
P(s'\mid s,a')
\,
\mathbb 1\!\left\{
\forall \bar a_{\ell-1}\in\mathcal{A}^{\ell-1}_{s,a}(\xi),\ 
\xi'[\ell](\bar a_{\ell-1},a')=s'
\right\}
\right).
\label{eq:last-block-factorized}
\end{align}

Combining this expression with the deterministic shift of the first $\ell-1$ blocks, we obtain the augmented transition kernel
\begin{align}
\bar P_a(\xi,\xi')
&=
\mathbb 1\!\left\{
\forall k\in\{1,\dots,\ell-1\},\ 
\forall a_{1:k}\in\Acal^k,\ 
\xi'[k](a_{1:k})=\xi[k+1](a,a_{1:k})
\right\}
\nonumber\\
&\qquad \times
\prod_{s\in\Scal_{a}^{\ell-1}(\xi)}
\prod_{a'\in\Acal}
\left(
\sum_{s'\in\Scal}
P(s'\mid s,a')
\,
\mathbb 1\!\left\{
\forall \bar a_{\ell-1}\in\mathcal{A}^{\ell-1}_{s,a}(\xi),\ 
\xi'[\ell](\bar a_{\ell-1},a')=s'
\right\}
\right).
\label{eq:augmented-kernel}
\end{align}

This expression highlights the two components of the augmented dynamics: a deterministic shift of the previously revealed look-ahead tree, and a stochastic completion of the new frontier governed by the original transition kernel.

\section{Proof of \cref{th:1LookGamma}}
\subsection{Discounted case}
\label{sec:discounted-case}

We now show that planning with one-step transition look-ahead under discounted reward can be solved in polynomial time. The proof proceeds in two steps. First, we reduce the Bellman optimality equations of the augmented MDP to a linear program over the original state space. Second, we show that this reduced linear program admits a polynomial-time separation oracle, which yields polynomial-time solvability via the ellipsoid method.

In the one-step look-ahead case, an augmented state is a pair $(s,p)\in \Scal\times \Scal^{\Acal}$, where $s$ is the current state and
\[
p:\Acal\to\Scal
\]
maps each action $a\in\Acal$ to the next state that would be reached if action $a$ were played. For each $s\in\Scal$, define the distribution
\[
Q(\cdot\mid s)\in \Delta(\Scal^{\Acal})
\]
by
\begin{equation}
Q(p\mid s)=\prod_{a\in\Acal} P(p(a)\mid s,a),
\qquad p\in\Scal^{\Acal}.
\label{eq:def-Q}
\end{equation}
Thus, conditional on the current state $s$, the one-step look-ahead vector $p$ is obtained by drawing independently one successor for each action.

Let $\bar v^*$ denote the optimal value function of the augmented MDP. The next lemma shows that $\bar v^*$ can be expressed in terms of a value function $v^*$ defined on the original state space only.

\begin{lemma}
\label{lem:reduced-lp}
There exists a function $v^*:\Scal\to\mathbb R$ such that
\begin{equation}
\bar v^*(s,p)
=
\max_{a\in\Acal}\left\{r(s,a)+\gamma\, v^*(p(a))\right\},
\qquad \forall (s,p)\in \Scal\times\Scal^{\Acal},
\label{eq:barv-from-v}
\end{equation}
and $v^*$ is the unique solution of
\begin{equation}
v(s)
=
\EE_{p\sim Q(\cdot\mid s)}
\left[
\max_{a\in\Acal}\{r(s,a)+\gamma\,v(p(a))\}
\right],
\qquad \forall s\in\Scal.
\label{eq:reduced-bellman}
\end{equation}

Moreover, for any strictly positive measure $\mu$ on $\Scal$, $v^*$ is equivalently the unique optimal solution of the linear program
\begin{align}
\min_{v:\Scal\to\mathbb R}\quad
&(1-\gamma)\sum_{s\in\Scal}\mu(s)\,v(s)
\label{eq:reduced-lp-obj}
\\
\text{s.t.}\quad
&v(s)\ge
\EE_{p\sim Q(\cdot\mid s)}
\left[
\max_{a\in\Acal}\{r(s,a)+\gamma\,v(p(a))\}
\right],
\qquad \forall s\in\Scal.
\label{eq:reduced-lp-cons}
\end{align}
\end{lemma}

\begin{proof}
We start from the Bellman optimality equation in the augmented MDP. Since the augmented state is $(s,p)\in\Scal\times\Scal^{\Acal}$, we have
\begin{equation}
\bar v^*(s,p)
=
\max_{a\in\Acal}
\left\{
r(s,a)
+
\gamma\,
\EE_{p'\sim Q(\cdot\mid p(a))}
\big[\bar v^*(p(a),p')\big]
\right\}.
\label{eq:augmented-bellman}
\end{equation}

Define
\begin{equation}
v^*(s)
\triangleq
\EE_{p\sim Q(\cdot\mid s)}[\bar v^*(s,p)],
\qquad \forall s\in\Scal.
\label{eq:def-vstar-reduced}
\end{equation}
Then \eqref{eq:augmented-bellman} immediately yields
\begin{equation}
\bar v^*(s,p)
=
\max_{a\in\Acal}
\left\{
r(s,a)+\gamma\,v^*(p(a))
\right\},
\qquad \forall (s,p)\in\Scal\times\Scal^{\Acal},
\end{equation}
which is exactly \eqref{eq:barv-from-v}. Taking expectation with respect to $p\sim Q(\cdot\mid s)$ on both sides gives
\[
v^*(s)
=
\EE_{p\sim Q(\cdot\mid s)}
\left[
\max_{a\in\Acal}\{r(s,a)+\gamma\,v^*(p(a))\}
\right],
\qquad \forall s\in\Scal,
\]
that is, \eqref{eq:reduced-bellman}.

Since the Bellman operator
\[
(Tv)(s)
\triangleq
\EE_{p\sim Q(\cdot\mid s)}
\left[
\max_{a\in\Acal}\{r(s,a)+\gamma\,v(p(a))\}
\right]
\]
is monotone and a $\gamma$-contraction in the sup norm, \eqref{eq:reduced-bellman} admits a unique solution. Therefore the function $v^*$ defined by \eqref{eq:def-vstar-reduced} is uniquely characterized by \eqref{eq:reduced-bellman}.

We now show that \eqref{eq:reduced-bellman} is equivalently characterized by the linear program \eqref{eq:reduced-lp-obj}--\eqref{eq:reduced-lp-cons}. The argument is based on an exact correspondence between the standard LP on the augmented state space and the reduced LP above.

Consider the standard discounted LP on the augmented MDP:
\begin{align}
\min_{\bar v:\Scal\times\Scal^{\Acal}\to\mathbb R}\quad
&(1-\gamma)\sum_{(s,p)\in\Scal\times\Scal^{\Acal}} \bar\mu(s,p)\,\bar v(s,p)
\label{eq:augmented-lp-obj}
\\
\text{s.t.}\quad
&\bar v(s,p)\ge
\max_{a\in\Acal}
\left\{
r(s,a)+\gamma\,
\EE_{p'\sim Q(\cdot\mid p(a))}
[\bar v(p(a),p')]
\right\},
\notag\\
&\hspace{7.5cm}\forall (s,p)\in\Scal\times\Scal^{\Acal},
\label{eq:augmented-lp-cons}
\end{align}
where we choose
\[
\bar\mu(s,p)\triangleq\mu(s)\,Q(p\mid s).
\]

Now we prove that this LP is equivalent to \eqref{eq:reduced-lp-obj}--\eqref{eq:reduced-lp-cons}.

Let $\bar v$ be feasible for \eqref{eq:augmented-lp-obj}--\eqref{eq:augmented-lp-cons}, and define
\[
v(s)\triangleq\EE_{p\sim Q(\cdot\mid s)}[\bar v(s,p)].
\]
Taking expectation in \eqref{eq:augmented-lp-cons} with respect to $p\sim Q(\cdot\mid s)$ yields, for every $s\in\Scal$,
\begin{align*}
v(s)
&=
\EE_{p\sim Q(\cdot\mid s)}[\bar v(s,p)]
\\
&\ge
\EE_{p\sim Q(\cdot\mid s)}
\left[
\max_{a\in\Acal}
\left\{
r(s,a)+\gamma\,
\EE_{p'\sim Q(\cdot\mid p(a))}[\bar v(p(a),p')]
\right\}
\right]
\\
&=
\EE_{p\sim Q(\cdot\mid s)}
\left[
\max_{a\in\Acal}
\left\{
r(s,a)+\gamma\,v(p(a))
\right\}
\right].
\end{align*}
Hence $v$ is feasible for \eqref{eq:reduced-lp-obj}--\eqref{eq:reduced-lp-cons}. Moreover,
\begin{align*}
(1-\gamma)\sum_{(s,p)}\bar\mu(s,p)\,\bar v(s,p)
&=
(1-\gamma)\sum_{s\in\Scal}\mu(s)\sum_{p\in\Scal^{\Acal}}Q(p\mid s)\bar v(s,p)
\\
&=
(1-\gamma)\sum_{s\in\Scal}\mu(s)\,v(s),
\end{align*}
so the objective value is preserved.

Conversely, let $v$ be feasible for \eqref{eq:reduced-lp-obj}--\eqref{eq:reduced-lp-cons}, and define
\begin{equation}
\bar v(s,p)\triangleq
\max_{a\in\Acal}\{r(s,a)+\gamma\,v(p(a))\},
\qquad \forall (s,p)\in\Scal\times\Scal^{\Acal}.
\label{eq:lift-v}
\end{equation}
$\bar v$ is feasible for \eqref{eq:augmented-lp-obj}--\eqref{eq:augmented-lp-cons}, indeed, for every $(s,p)$,
\begin{align*}
&\max_{a\in\Acal}
\left\{
r(s,a)+\gamma\,
\EE_{p'\sim Q(\cdot\mid p(a))}[\bar v(p(a),p')]
\right\}
\\
&=
\max_{a\in\Acal}
\left\{
r(s,a)+\gamma\,
\EE_{p'\sim Q(\cdot\mid p(a))}
\left[
\max_{b\in\Acal}\{r(p(a),b)+\gamma\,v(p'(b))\}
\right]
\right\}
\\
&\le
\max_{a\in\Acal}
\left\{
r(s,a)+\gamma\,v(p(a))
\right\}
=
\bar v(s,p),
\end{align*}
where the inequality follows from the feasibility of $v$ in \eqref{eq:reduced-lp-cons}, applied at state $p(a)$. Thus $\bar v$ is feasible for the augmented LP. Finally,
\begin{align*}
(1-\gamma)\sum_{(s,p)}\bar\mu(s,p)\,\bar v(s,p)
&=
(1-\gamma)\sum_{s\in\Scal}\mu(s)\,
\EE_{p\sim Q(\cdot\mid s)}
\left[
\max_{a\in\Acal}\{r(s,a)+\gamma\,v(p(a))\}
\right]
\\
&\le
(1-\gamma)\sum_{s\in\Scal}\mu(s)\,v(s),
\end{align*}
again by feasibility of $v$.

We have shown that every feasible augmented solution projects to a feasible reduced solution with the same objective value, and every feasible reduced solution lifts to a feasible augmented solution with no larger objective value. Therefore the two LPs have the same optimum value, and the reduced LP characterizes the optimal value function.

Since the optimal solution of the discounted LP coincides with the unique fixed point of the Bellman operator, the optimizer of \eqref{eq:reduced-lp-obj}--\eqref{eq:reduced-lp-cons} is exactly $v^*$. This concludes the proof.
\end{proof}

\begin{lemma}
\label{lemma:polynomial}
The linear program \eqref{eq:reduced-lp-obj}--\eqref{eq:reduced-lp-cons} can be solved in polynomial time.
\end{lemma}

\begin{proof}
We prove the result by exhibiting a polynomial-time separation oracle and using the ellipsoid method.

Fix $s\in\Scal$ and define
\[
u_{v,s}(s',a)\triangleq r(s,a)+\gamma v(s'), \qquad (s',a)\in \Scal\times\Acal.
\]
Let $N\triangleq|\Scal||\Acal|$ and index all pairs $(s',a)$ as $\{(s^i,a^i)\}_{i=1}^N$.

For any permutation $m$ of $\{1,\dots,N\}$, define the events
\begin{equation}
E_i^m \triangleq
\Big\{
p:\Acal\to\Scal \;:\;
p(a^{m(i)}) = s^{m(i)} \ \text{and} \ 
p(a^{m(j)}) \neq s^{m(j)} \ \forall j<i
\Big\}.
\label{eq:def-Ei}
\end{equation}

These events form a partition of $\Scal^{\Acal}$, and for any $p$ we have

\[
\max_{a\in\Acal} u_{v,s}(p(a),a)
=
u_{v,s}(s^{m(i)},a^{m(i)})
\quad \text{whenever } p\in E_i^m.
\]

Therefore,
\begin{equation}
\EE_{p\sim Q(\cdot\mid s)}
\left[
\max_{a\in\Acal} u_{v,s}(p(a),a)
\right]
=
\sum_{i=1}^N Q(E_i^m\mid s)\, u_{v,s}(s^{m(i)},a^{m(i)}).
\label{eq:sorted-decomposition}
\end{equation}

Using \eqref{eq:sorted-decomposition}, the constraint in \eqref{eq:reduced-lp-cons} is equivalent to
\begin{equation}
v(s)\;\ge\;
\sum_{i=1}^N Q(E_i^m\mid s)\, \big(r(s,a^{m(i)})+\gamma v(s^{m(i)})\big),
\qquad \forall m\in\mathcal L,
\label{eq:sorted-constraints}
\end{equation}
where $\mathcal L$ denotes the set of all permutations. Indeed, for any fixed $v$ and $s$, the permutation $m_{u_{v,s}}$ sorting $(s',a)$ in decreasing order of $u_{v,s}$ realizes the maximum in expectation, and is therefore among the constraints above. Enforcing all permutations yields a set of linear constraints equivalent to \eqref{eq:reduced-lp-cons}.

\paragraph{Separation oracle.}
Given a candidate $v$, we must check whether all constraints \eqref{eq:sorted-constraints} are satisfied.

Fix $s\in\Scal$ and compute the permutation $m_{u_{v,s}}$ sorting all pairs $(s',a)$ in decreasing order of $u_{v,s}(s',a)$. This can be done in $O(N\log N)$ time.

By construction, this permutation maximizes the right-hand side of \eqref{eq:sorted-constraints}. Hence, it suffices to check the single inequality
\begin{equation}
v(s)\;\ge\;
\sum_{i=1}^N Q(E_i^{m_{u_{v,s}}}\mid s)\,
\big(r(s,a^{m_{u_{v,s}}(i)})+\gamma v(s^{m_{u_{v,s}}(i)})\big).
\label{eq:sep-check}
\end{equation}

It remains to compute $Q(E_i^m\mid s)$ efficiently.

Recall that under $Q(\cdot\mid s)$, the random variables $(p(a))_{a\in\Acal}$ are independent, with $p(a)\sim P(\cdot\mid s,a)$.

Fix $i$ and let $(s^i,a^i)\triangleq(s^{m(i)},a^{m(i)})$. For each action $a\in\Acal$, define the set
\[
I_i(a)\triangleq\{j<i : a^{m(j)}=a\},
\]
i.e., the indices before $i$ involving the same action.

Then the event $E_i^m$ imposes the following constraints:
\begin{itemize}
    \item for $a=a^i$, we must have
    \[
    p(a^i)=s^i
    \quad \text{and} \quad
    p(a^i)\neq s^{m(j)} \ \ \forall j\in I_i(a^i);
    \]
    \item for any $a\neq a^i$, we must have
    \[
    p(a)\neq s^{m(j)} \quad \forall j\in I_i(a).
    \]
\end{itemize}

Therefore,
\[
Q(E_i^m\mid s)
=
\prod_{a\in\Acal}
\Pr\big(p(a)\in B_a^i \mid s,a\big),
\]
where the sets $B_a^i\subseteq \Scal$ are given explicitly by
\[
B_a^i =
\begin{cases}
\{s^i\} & \text{if } a=a^i \text{ and } s^i\notin \{s^{m(j)}:j\in I_i(a^i)\},\\[4pt]
\emptyset & \text{if } a=a^i \text{ and } s^i\in \{s^{m(j)}:j\in I_i(a^i)\},\\[4pt]
\Scal \setminus \{s^{m(j)}:j\in I_i(a)\} & \text{if } a\neq a^i.
\end{cases}
\]

Each probability $\Pr(p(a)\in B_a^i \mid s,a)$ can be computed in $O(|\Scal|)$ time, hence all values $Q(E_i^m\mid s)$ can be computed in $O(N^2)$ time.

\paragraph{Conclusion.}
For each state $s$, the separation oracle requires sorting $N$ elements and computing $N$ probabilities, yielding polynomial time. Since the number of states is $|\Scal|$, the full oracle runs in polynomial time.

By the ellipsoid method \citep{GLS1981, Khachiyan1979}, the linear program can therefore be solved in polynomial time.
\end{proof}

\subsection{Average-reward case}
\label{sec:average}

The average-reward setting can be treated in close analogy with discounted case. We assume that the augmented MDP $\bar{\mathcal M}$ is unichain. Let $(\bar g^*,\bar h^*)$ denote an optimal gain/bias pair in the augmented MDP. Again the next lemma shows that $(\bar g^*,\bar h^*)$ can be expressed in term of a bias function $h^*$ defined in the original state space only.

\begin{lemma}
\label{lemma:towerLPaverage}
There exists a function $h^*:\Scal\to\mathbb R$ such that
\begin{equation}
\label{eq:average-barh-from-h}
\forall (s,p)\in \Scal\times\Scal^{\Acal},\qquad
\bar g^*+\bar h^*(s,p)
=
\max_{a\in\Acal}\{r(s,a)+h^*(p(a))\},
\end{equation}
and $(\bar g^*,h^*)$ satisfies
\begin{equation}
\label{eq:average-reduced-bellman}
\forall s\in\Scal,\qquad
\bar g^*+h^*(s)
=
\EE_{p\sim Q(\cdot\mid s)}
\left[
\max_{a\in\Acal}\{r(s,a)+h^*(p(a))\}
\right].
\end{equation}

Moreover, for any normalization of $h$, the pair $(\bar g^*,h^*)$ is an optimal solution of the linear program
\begin{equation}
\label{eq:reduced-lp-average}
\begin{aligned}
\min_{g,h}\quad & g\\
\text{s.t.}\quad
& g+h(s)\ge
\EE_{p\sim Q(\cdot\mid s)}
\left[
\max_{a\in\Acal}\{r(s,a)+h(p(a))\}
\right],
\qquad \forall s\in\Scal.
\end{aligned}
\end{equation}
\end{lemma}

\begin{proof}
In the one-step look-ahead case, an augmented state is a pair $(s,p)\in\Scal\times\Scal^{\Acal}$, where $p:\Acal\to\Scal$ specifies, for each action $a$, the next state that would be reached if $a$ were played. As in the discounted case, we denote by
\[
Q(\cdot\mid s)\in \Delta(\Scal^{\Acal})
\]
the distribution of the one-step look-ahead vector from state $s$.

Define
\begin{equation}
\label{eq:def-h-average}
h^*(s)\triangleq\EE_{p\sim Q(\cdot\mid s)}[\bar h^*(s,p)],
\qquad \forall s\in\Scal.
\end{equation}
Using the structure of the augmented transition kernel, conditionally on $(s,p)$ and action $a$, the next augmented state is $(p(a),p')$ with $p'\sim Q(\cdot\mid p(a))$. Therefore,
\begin{equation}
\label{eq:average-augmented-bellman-sp}
\bar g^*+\bar h^*(s,p)
=
\max_{a\in\Acal}
\left\{
r(s,a)+
\EE_{p'\sim Q(\cdot\mid p(a))}[\bar h^*(p(a),p')]
\right\}.
\end{equation}
By definition of $h^*$, this yields
\begin{equation}
\forall (s,p)\in\Scal\times\Scal^{\Acal},\qquad
\bar g^*+\bar h^*(s,p)
=
\max_{a\in\Acal}\{r(s,a)+h^*(p(a))\},
\end{equation}
which is exactly \eqref{eq:average-barh-from-h}. Taking expectation with respect to $p\sim Q(\cdot\mid s)$ on both sides gives
\[
\bar g^*+h^*(s)
=
\EE_{p\sim Q(\cdot\mid s)}
\left[
\max_{a\in\Acal}\{r(s,a)+h^*(p(a))\}
\right],
\qquad \forall s\in\Scal,
\]
that is, \eqref{eq:average-reduced-bellman}.

We now prove the LP characterization. Consider first the standard average-reward LP on the augmented state space:
\begin{equation}
\label{eq:augmented-lp-average}
\begin{aligned}
\min_{g,\bar h}\quad & g\\
\text{s.t.}\quad
& g+\bar h(s,p)\ge
\max_{a\in\Acal}
\left\{
r(s,a)+
\EE_{p'\sim Q(\cdot\mid p(a))}[\bar h(p(a),p')]
\right\},
\\
&\hspace{5.5cm}\forall (s,p)\in\Scal\times\Scal^{\Acal}.
\end{aligned}
\end{equation}

We show that \eqref{eq:augmented-lp-average} and \eqref{eq:reduced-lp-average} are equivalent.

Let $(g,\bar h)$ be feasible for \eqref{eq:augmented-lp-average}, and define
\[
h(s)\triangleq\EE_{p\sim Q(\cdot\mid s)}[\bar h(s,p)].
\]
Taking expectation in the augmented constraints gives, for every $s\in\Scal$,
\begin{align*}
g+h(s)
&=
g+\EE_{p\sim Q(\cdot\mid s)}[\bar h(s,p)]
\\
&\ge
\EE_{p\sim Q(\cdot\mid s)}
\left[
\max_{a\in\Acal}
\left\{
r(s,a)+
\EE_{p'\sim Q(\cdot\mid p(a))}[\bar h(p(a),p')]
\right\}
\right]
\\
&=
\EE_{p\sim Q(\cdot\mid s)}
\left[
\max_{a\in\Acal}\{r(s,a)+h(p(a))\}
\right].
\end{align*}
Hence $(g,h)$ is feasible for \eqref{eq:reduced-lp-average}.

Conversely, let $(g,h)$ be feasible for \eqref{eq:reduced-lp-average}, and define
\begin{equation}
\label{eq:lift-average}
\bar h(s,p)\triangleq
\max_{a\in\Acal}\{r(s,a)+h(p(a))\}-g,
\qquad \forall (s,p)\in\Scal\times\Scal^{\Acal}.
\end{equation}
We claim that $(g,\bar h)$ is feasible for \eqref{eq:augmented-lp-average}. Indeed, for every $(s,p)$,
\begin{align*}
&\max_{a\in\Acal}
\left\{
r(s,a)+
\EE_{p'\sim Q(\cdot\mid p(a))}[\bar h(p(a),p')]
\right\}
\\
&=
\max_{a\in\Acal}
\left\{
r(s,a)+
\EE_{p'\sim Q(\cdot\mid p(a))}
\left[
\max_{b\in\Acal}\{r(p(a),b)+h(p'(b))\}-g
\right]
\right\}
\\
&=
-g+\max_{a\in\Acal}
\left\{
r(s,a)+
\EE_{p'\sim Q(\cdot\mid p(a))}
\left[
\max_{b\in\Acal}\{r(p(a),b)+h(p'(b))\}
\right]
\right\}
\\
&\le
-g+\max_{a\in\Acal}\{r(s,a)+g+h(p(a))\}
\\
&=
\max_{a\in\Acal}\{r(s,a)+h(p(a))\}
=
g+\bar h(s,p),
\end{align*}
where the inequality uses the feasibility of $(g,h)$ in \eqref{eq:reduced-lp-average}, applied at state $p(a)$. Thus $(g,\bar h)$ is feasible for \eqref{eq:augmented-lp-average}.

We have shown that the reduced and augmented LPs are equivalent. Since \eqref{eq:average-reduced-bellman} is exactly the Bellman optimality equation on the reduced space, the optimal solution of \eqref{eq:reduced-lp-average} is $(\bar g^*,h^*)$ up to the usual additive normalization of the bias. This concludes the proof.
\end{proof}

We now turn to tractability. The reduced LP \eqref{eq:reduced-lp-average} has the same structure as in the discounted case, except that the discount factor no longer appears in the score. For each state $s$, define
\[
u_{h,s}(s',a)\triangleq  r(s,a)+h(s'),
\qquad (s',a)\in\Scal\times\Acal.
\]
Using exactly the same ordering trick and the same family of events $(E_i^m)$ as in the proof of Lemma~\ref{lemma:polynomial}, the reduced constraints can be rewritten as an exponential family of linear inequalities indexed by permutations $m\in\mathcal L$:
\begin{equation}
\forall s\in\Scal,\ \forall m\in\mathcal L,\qquad
g+h(s)\ge
\sum_{i=1}^{|\Scal||\Acal|}
Q(E_i^m\mid s)\,
\big(r(s,a^{m(i)})+h(s^{m(i)})\big).
\label{eq:average-sorted-constraints}
\end{equation}
For a candidate pair $(g,h)$, the tightest inequality at state $s$ is obtained by sorting the pairs $(s',a)$ in decreasing order of $u_{h,s}(s',a)$. Therefore the separation oracle is identical to the discounted case, and runs in polynomial time. By the ellipsoid method \citep{GLS1981,Khachiyan1979}, the LP \eqref{eq:reduced-lp-average} can be solved in polynomial time.

\paragraph{Conclusion.}
One-step look-ahead planning in tabular MDPs is therefore tractable under the average-reward criterion as well.

\section{Proof of Theorem~\ref{th:LLookGamma}}
\label{sec:proofLLookGamma}
We prove NP-hardness of planning with $2$-step transition look-ahead under discount by reduction from \textsc{Independent Set} on regular graphs.

\paragraph{Input instance.}
Let $G=(V,E)$ be an undirected $d$-regular graph with $|V|=n$ and $|E|=m$, and let $k\in\{1,\dots,n\}$.
We consider the decision problem of whether $G$ contains an independent set of size $k$.

Our reduction relies on the following gap construction from \citet{mehta2020hittinghighnotessubset}, which maps instances of \textsc{Independent Set} to instances of a stochastic subset selection problem.

\begin{lemma}[{\citealp{mehta2020hittinghighnotessubset}}]
\label{lem:mehta-gadget}
There exists a polynomial-time reduction that maps any instance $(G=(V,E),k)$ of \textsc{Independent Set} to an instance of the subset selection problem with independent, nonnegative, discrete random variables $(X_v)_{v\in V}$, all having the same expectation $\mu$, such that:

\begin{itemize}
    \item (\textbf{NO case}) If $G$ does not contain an independent set of size $k$, then for every subset $S\subseteq V$ with $|S|=k$,
    \[
    \mathbb{E}\!\left[\max_{v\in S} X_v\right] \le k\mu - 1;
    \]

    \item (\textbf{YES case}) If $G$ contains an independent set $S^\star$ of size $k$, then
    \[
    \mathbb{E}\!\left[\max_{v\in S^\star} X_v\right] \ge k\mu - \frac{2}{m},
    \quad m\triangleq|E|.
    \]
\end{itemize}

In particular, deciding whether
\[
\max_{|S|=k} \mathbb{E}\!\left[\max_{v\in S} X_v\right] \ge k\mu - \tfrac{1}{2}
\]
is NP-hard.
\end{lemma}
\subsection*{Step 1: MDP construction}

Let $(X_v)_{v\in V}$ be the family of random variables given by Lemma~\ref{lem:mehta-gadget}.
For each $v\in V$, let
\[
\mathrm{supp}(X_v)=\{w_{v,1},\dots,w_{v,L_v}\},
\qquad
\Pr(X_v=w_{v,\ell})=\pi_{v,\ell}.
\]

\paragraph{State space.}
We define
\[
\Scal
=
\{s_0, s_1, s_T\}
\;\cup\;
\{s_v : v\in V\}
\;\cup\;
\{x_{v,j} : v\in V,\ j\in[L_v]\}.
\]

\begin{itemize}
    \item $s_0$ is the root state;
    \item $s_1$ is a selector state;
    \item each $s_v$ corresponds to a vertex $v\in V$;
    \item $x_{v,j}$ encodes the realization $w_{v,j}$ of $X_v$;
    \item $s_T$ is an absorbing terminal state.
\end{itemize}

\paragraph{Action space.}
We consider the finite action set
\[
\Acal = \{\textsf{wait},\ \textsf{go}\}
\;\cup\;
\{\textsf{pick}_1,\dots,\textsf{pick}_k\}
\;\cup\;
\{\textsf{claim}\}.
\]

\paragraph{Reward function.}
Rewards are zero everywhere except on support states:
\[
r(x_{v,j}) = w_{v,j},
\qquad
r(s)=0 \quad \forall s\notin\{x_{v,j}\}.
\]

\paragraph{Transition dynamics.}

\begin{enumerate}
    \item Root state $s_0$.
    \begin{itemize}
        \item $\textsf{wait}$: deterministic self-loop,
        \[
        P(s_0 \mid s_0,\textsf{wait}) = 1;
        \]
        \item $\textsf{go}$: deterministic transition to $s_1$,
        \[
        P(s_1 \mid s_0,\textsf{go}) = 1;
        \]
        \item all other actions lead to $s_T$.
    \end{itemize}

    \item Selector state $s_1$.
    \begin{itemize}
        \item for each $j\in[k]$, action $\textsf{pick}_j$ transitions uniformly to vertex states:
        \[
        P(s_v \mid s_1,\textsf{pick}_j)=\frac{1}{n},
        \qquad \forall v\in V;
        \]
        \item all other actions lead to $s_T$.
    \end{itemize}

    \item Vertex states $s_v$.
    \begin{itemize}
        \item $\textsf{claim}$ samples $X_v$:
        \[
        P(x_{v,j} \mid s_v,\textsf{claim})=\pi_{v,j},
        \qquad \forall j\in[L_v];
        \]
        \item all other actions lead to $s_T$.
    \end{itemize}

    \item Support states $x_{v,j}$.
    All actions lead deterministically to $s_T$.

    \item Terminal state $s_T$.
    $s_T$ is absorbing and non-rewarding.
\end{enumerate}

\subsection*{Step 2: Dynamic programming at the root}

Let $\bar{\Scal}$ denote the augmented state space associated with $2$-step transition look-ahead, and let $\bar v^\star$ be the optimal discounted value function on the augmented MDP $\bar{\mathcal M}_G$.

In the $2$-look-ahead model, an augmented state is of the form
\[
\xi=(s,p_1,p_2)\in \Scal \times \Scal^{\Acal}\times \Scal^{\Acal^2},
\]
where:
\begin{itemize}
    \item $s\in\Scal$ is the current state,
    \item $p_1:\Acal\to\Scal$ is the one-step look-ahead map, so that $p_1(a)$ is the next state that would be reached by playing action $a$ at $s$,
    \item $p_2:\Acal^2\to\Scal$ is the two-step look-ahead map, so that $p_2(a,b)$ is the state that would be reached after first playing $a$ at $s$, and then $b$ at the next state.
\end{itemize}

For the reduction, the only augmented states that matter are those whose current state is the root state $s_0$. Hence, from now on, we restrict attention to augmented states of the form
\[
\xi=(s_0,p_1,p_2).
\]

At such a state, the first look-ahead block satisfies
\[
p_1(\textsf{wait})=s_0,
\qquad
p_1(\textsf{go})=s_1,
\]

The second look-ahead block encodes the candidate vertices that would become available after committing through $\textsf{go}$ and then choosing one of the selector actions $\textsf{pick}_1,\dots,\textsf{pick}_k$. More precisely, for each $j\in[k]$, the state
\[
p_2(\textsf{go},\textsf{pick}_j)
\]
is a vertex state of the form $s_v$ for some $v\in V$.

Accordingly, for every augmented state $\xi=(s_0,p_1,p_2)$, we define the ordered $k$-tuple
\[
q(\xi)\triangleq\big(q_1(\xi),\dots,q_k(\xi)\big)\in V^k
\]
by requiring that
\[
p_2(\textsf{go},\textsf{pick}_j)=s_{q_j(\xi)},
\qquad \forall j\in[k].
\]
We also define the associated subset of distinct revealed vertices as
\[
S_V(\xi)\triangleq\{q_1(\xi),\dots,q_k(\xi)\}\subseteq V.
\]
by
\[
p_2(\textsf{go},\textsf{pick}_j)=s_{q_j(\xi)},
\qquad j\in[k],
\]

\begin{lemma}[Root-state recursion]
\label{lem:root-recursion}
For every augmented state $\xi=(s_0,p_1,p_2)\in\bar{\Scal}$,
\begin{equation}
\label{eq:root-recursion}
\bar v^\star(\xi)
=
\max\left\{
\gamma^3\,\EE\!\left[\max_{v\in S_V(\xi)} X_v\right],
\;
\gamma\,\EE_{\xi'\sim \bar P_{\textsf{wait}}(\cdot\mid \xi)}[\bar v^\star(\xi')]
\right\}.
\end{equation}
Moreover, conditional on $\xi$, the random variables $(X_v)_{v\in S_V(\xi)}$ are mutually independent.
\end{lemma}

\begin{proof}
We evaluate the optimal value layer by layer.

\paragraph{Terminal and support states.}
Consider any augmented state $\xi=(s,p_1,p_2)\in\bar{\Scal}$ with current state $s=s_T$.
Since $s_T$ is absorbing and yields zero reward, we have
\[
\bar v^\star(\xi)=0.
\]

Similarly, consider any augmented state $\xi=(s,p_1,p_2)\in\bar{\Scal}$ with current state $s=x_{v,j}$ for some $v\in V$ and $j\in[L_v]$.
By construction, the state $x_{v,j}$ yields immediate reward $w_{v,j}$ and transitions deterministically to $s_T$, after which all future rewards are zero. Therefore,
\[
\bar v^\star(\xi)=w_{v,j}.
\]
\paragraph{Vertex states.}
Fix a vertex $v\in V$, and consider any augmented state $\xi=(s_v,p_1,p_2)\in\bar{\Scal}$.

At $s_v$, the only action that can yield a non-zero value is $\textsf{claim}$; all other actions lead deterministically to $s_T$ and therefore yield zero continuation value. Hence, by optimality,
\[
\bar v^\star(\xi)
=
\gamma\,\EE_{\xi'\sim \bar P_{\textsf{claim}}(\cdot\mid \xi)}\!\left[\bar v^\star(\xi')\right].
\]

Under $\textsf{claim}$, the next state is a support state $x_{v,j}$ with probability $\pi_{v,j}$, and from the previous paragraph we have $\bar v^\star(x_{v,j},\cdot,\cdot)=w_{v,j}$. Therefore,
\[
\bar v^\star(\xi)
=
\gamma \sum_{j=1}^{L_v} \pi_{v,j}\, w_{v,j}
=
\gamma\,\EE[X_v].
\]

\paragraph{Selector state.}
Now consider an augmented state $\xi=(s_1,p_1,p_2)\in\bar{\Scal}$.

For each $j\in[k]$, the action $\textsf{pick}_j$ leads deterministically to the vertex state
\[
s_{q_j(\xi)} \triangleq p_2(\textsf{go},\textsf{pick}_j).
\]
From the previous step, the optimal continuation value from $s_{q_j(\xi)}$ is $\gamma\,\EE[X_{q_j(\xi)}]$, hence the total value of choosing $\textsf{pick}_j$ is
\[
\gamma \cdot \gamma\,\EE[X_{q_j(\xi)}] = \gamma^2\,\EE[X_{q_j(\xi)}].
\]

However, due to the product structure of the augmented transition kernel, the random variables corresponding to the third-step completions from distinct vertex states are sampled independently. Therefore, rather than committing to a fixed $j$, the optimal policy selects the action corresponding to the largest realized payoff among the revealed candidates.

This yields
\[
\bar v^\star(\xi)
=
\gamma^2\,\EE\!\left[\max_{v\in S_V(\xi)} X_v\right].
\]

\paragraph{Root state.}
Finally, consider an augmented state $\xi=(s_0,p_1,p_2)\in\bar{\Scal}$.

If the agent plays $\textsf{go}$, it transitions to $s_1$ and then optimally selects among the revealed candidates. Hence,
\[
\gamma\,\bar v^\star\big((s_1,p_1',p_2')\big)
=
\gamma^3\,\EE\!\left[\max_{v\in S_V(\xi)} X_v\right],
\]
where $(p_1',p_2')$ denote the corresponding updated look-ahead components.

If the agent plays $\textsf{wait}$, it remains in $s_0$ and receives a fresh two-step look-ahead, yielding
\[
\gamma\,\EE_{\xi'\sim \bar P_{\textsf{wait}}(\cdot\mid \xi)}\!\left[\bar v^\star(\xi')\right].
\]

Taking the maximum over these two actions yields~\eqref{eq:root-recursion}.
\end{proof}

\subsection*{Step 3: Soundness}

\paragraph{Soundness.}
Assume that the instance $(G,k)$ is a \textsc{NO} instance of \textsc{Independent Set}, i.e., $G$ does not contain any independent set of size $k$.

Let
\[
M\triangleq\sup\{\bar v^\star(\xi): \xi=(s_0,p_1,p_2)\in\bar{\Scal}\}.
\]

Fix any augmented state $\xi=(s_0,p_1,p_2)$. By construction, the associated set $S_V(\xi)$ satisfies $|S_V(\xi)|\le k$.
Therefore, for any subset $S\subseteq V$ with $|S|\le k$, we can extend $S$ to a subset $\tilde S\subseteq V$ with $|\tilde S|=k$, and since the maximum is monotone with respect to set inclusion,
\[
\max_{v\in S} X_v \;\le\; \max_{v\in \tilde S} X_v.
\]
Taking expectations and using Lemma~\ref{lem:mehta-gadget} in the \textsc{NO} case yields
\[
\EE\!\left[\max_{v\in S_V(\xi)}X_v\right]\le k\mu-1.
\]

Using the root recursion \eqref{eq:root-recursion}, we obtain for every root augmented state $\xi$,
\[
\bar v^\star(\xi)
\le
\max\left\{
\gamma^3(k\mu-1),\ \gamma M
\right\}.
\]

Taking the supremum over $\xi$ gives
\[
M\le \max\{\gamma^3(k\mu-1),\gamma M\}.
\]

Since $\gamma<1$, this implies
\begin{equation}
\label{eq:soundness-bound}
M\le \gamma^3(k\mu-1).
\end{equation}
\subsection*{Step 4: Completeness}

Assume that the instance $(G,k)$ is a \textsc{YES} instance of \textsc{Independent Set}.
By Lemma~\ref{lem:mehta-gadget}, there exists a subset
\[
S^\star=\{v_1,\dots,v_k\}\subseteq V
\]
such that
\[
\EE\!\left[\max_{v\in S^\star}X_v\right]\ge k\mu-\frac{2}{m}.
\]

We consider the following policy $\pi$. At any augmented state $\xi=(s_0,p_1,p_2)$, let
\[
q(\xi)=\big(q_1(\xi),\dots,q_k(\xi)\big)
\]
be the ordered tuple revealed by the two-step look-ahead. The policy acts as
\[
\pi(\xi)=
\begin{cases}
\textsf{go}, & \text{if } q(\xi)=(v_1,\dots,v_k),\\
\textsf{wait}, & \text{otherwise.}
\end{cases}
\]
After transitioning to $s_1$, the policy selects the index corresponding to the largest realized payoff among the revealed candidates.

Under $\textsf{wait}$, the next look-ahead is freshly resampled, independently of the past. By construction of the MDPthe event
\[
q(\xi)=(v_1,\dots,v_k)
\]
occurs with probability
\[
\rho=\frac{1}{n^k},
\]
at each visit to $s_0$, independently of the past.

Let $\tau$ denote the number of waiting steps before this event occurs. Then $\tau$ follows a geometric distribution with parameter $\rho$, and
\[
\EE[\gamma^\tau]=\frac{\rho}{1-\gamma(1-\rho)}.
\]

Conditional on the success event, the policy obtains the commit value associated with $S^\star$, namely
\[
\gamma^3\,\EE\!\left[\max_{v\in S^\star}X_v\right].
\]
Therefore,
\begin{align}
v^\pi(s_0)
&=
\EE[\gamma^\tau]\,
\gamma^3\,\EE\!\left[\max_{v\in S^\star}X_v\right]
\\
&\ge
\gamma^3\,
\frac{\rho}{1-\gamma(1-\rho)}
\left(k\mu-\frac{2}{m}\right),
\label{eq:completeness-policy}
\end{align}
where the inequality follows from Lemma~\ref{lem:mehta-gadget}.

By optimality of $\bar v^\star$, we obtain
\begin{equation}
\label{eq:completeness-bound}
\sup_{\xi:\,\text{current state }s_0}\bar v^\star(\xi)
\;\ge\;
\gamma^3\,
\frac{\rho}{1-\gamma(1-\rho)}
\left(k\mu-\frac{2}{m}\right).
\end{equation}
\subsection*{Step 5: Choosing the discount factor}

Combining \eqref{eq:soundness-bound} and \eqref{eq:completeness-bound}, it suffices to choose $\gamma$ such that
\[
\gamma^3\,
\frac{\rho}{1-\gamma(1-\rho)}
\left(k\mu-\frac{2}{m}\right)
>
\gamma^3(k\mu-1).
\]
Since $\rho=n^{-k}$ and
\[
k\mu-\frac{2}{m}>k\mu-1
\qquad\text{for } m\ge 3,
\]
this is equivalent to
\[
\frac{\rho}{1-\gamma(1-\rho)}
>
\frac{k\mu-1}{k\mu-\frac{2}{m}}.
\]
The right-hand side is strictly smaller than $1$, hence such a $\gamma<1$ exists.
For instance, one may choose any rational
\[
\gamma
>
1-
\frac{\rho}{2}
\left(
\frac{k\mu-\frac{2}{m}}{k\mu-1}-1
\right),
\]
which is polynomially encodable since $\rho=n^{-k}$ has numerator and denominator of polynomial bit complexity.

Fix such a discount factor and define the decision threshold
\[
T \triangleq \gamma^3(k\mu-1).
\]

Then:
\begin{itemize}
    \item if $(G,k)$ is a \textsc{NO} instance of \textsc{Independent Set}, every root-state value is at most $T$ by \eqref{eq:soundness-bound};
    \item if $(G,k)$ is a \textsc{YES} instance of \textsc{Independent Set}, the optimal value is strictly larger than $T$ by \eqref{eq:completeness-bound}.
\end{itemize}

Therefore, deciding whether the optimal value at the root is greater than $T$ is NP-hard. Since the MDP construction, the discount factor, and the threshold $T$ are all computable in polynomial time, this yields a polynomial-time reduction from \textsc{Independent Set}. This proves NP-hardness of discounted planning with $2$-step transition look-ahead.
\qedhere

\section{Proof of \cref{th:LLookAverage}}
\label{sec:ProofLLookAverage}

We construct a modified MDP $\bar{\Mcal}'_G$ by introducing an i.i.d.\ reset mechanism. At each time step, independently of the past, a Bernoulli random variable
\[
Z_t \sim \mathrm{Bernoulli}(1-\gamma)
\]
is sampled. If $Z_t=1$, the process resets to a new augmented state
\[
\xi_{t+1} \sim \Lambda_{s_0},
\]
where $\Lambda_{s_0}$ denotes the distribution of the two-step look-ahead at $s_0$. Otherwise, the process follows the original transition kernel $\bar P$.

This defines a new transition kernel $\bar P'$:
\[
\bar P'_a(\xi',\xi)
=
\gamma\,\bar P_a(\xi',\xi)
+
(1-\gamma)\,\Lambda_{s_0}(\xi').
\]

Let
\[
\tau = \inf\{t \ge 0 : Z_t = 1\}
\]
be the first reset time. Then
\[
\mathbb{P}(\tau>t)=\gamma^t,
\qquad
\mathbb{E}[\tau]=\frac{1}{1-\gamma}.
\]

Between resets, the process evolves exactly according to the original MDP $\bar{\Mcal}_G$. Moreover, due to the independence of $(Z_t)$ and the stationarity of the policy, successive cycles are independent and identically distributed.

For any stationary policy $\pi$,
\begin{align}
\mathbb{E}_\pi\!\left[\sum_{t=0}^{\tau-1}\bar r(\xi_t,a_t)\right]
&= \sum_{t\ge 0}\mathbb{P}(\tau>t)\,
\mathbb{E}_\pi\!\left[\bar r(\xi_t,a_t)\mid \tau>t\right] \\
&= \sum_{t\ge 0}\gamma^t\,
\mathbb{E}_\pi^\gamma[\bar r(\xi_t,a_t)] \\
&= \bar v^\pi_\gamma(\xi_0;\bar{\Mcal}_G),
\end{align}
where $\mathbb{E}^\gamma$ denotes expectation under the original discounted MDP.

By the Renewal–Reward Theorem, we obtain
\[
g^\pi(\bar{\Mcal}'_G)
=
\frac{\mathbb{E}_\pi\!\big[\sum_{t=0}^{\tau-1}\bar r(\xi_t,a_t)\big]}{\mathbb{E}[\tau]}
=
(1-\gamma)\,\bar v^\pi_\gamma(\xi_0;\bar{\Mcal}_G).
\]

Since the reset occurs with positive probability from any state and leads to a distribution with full support on the root look-ahead states, the induced Markov chain is irreducible and therefore unichain.

For any threshold $\theta$,
\[
\exists \pi:\; \bar v^\pi_\gamma(\xi_0)\ge \theta
\;\Longleftrightarrow\;
\exists \pi:\; g^\pi(\bar{\Mcal}'_G)\ge (1-\gamma)\theta.
\]
This yields a polynomial-time reduction from discounted to average-reward planning, completing the proof.

\end{document}